\documentclass{article} % DO NOT CHANGE THIS
\usepackage{amsmath,amsfonts,bm}

\def\eqref#1{equation~\ref{#1}}
\def\ceil#1{\lceil #1 \rceil}

\def\1{\bm{1}}

\def\eps{{\epsilon}}

\DeclareMathAlphabet{\mathsfit}{\encodingdefault}{\sfdefault}{m}{sl}
\SetMathAlphabet{\mathsfit}{bold}{\encodingdefault}{\sfdefault}{bx}{n}

\newcommand{\E}{\mathbb{E}}

\newcommand{\R}{\mathbb{R}}

\usepackage[margin=1.02in]{geometry}
\usepackage{amsmath,amssymb,amsfonts}
\usepackage{amsthm}
\usepackage{chngcntr}
\usepackage{helvet}  % DO NOT CHANGE THIS
\usepackage{courier}  % DO NOT CHANGE THIS
\usepackage{graphicx}
\usepackage{textcomp}
\usepackage{setspace}
\usepackage{algpseudocode}

\theoremstyle{plain}
\newtheorem{theorem}{Theorem}[section]
\newtheorem{proposition}{Proposition}
\newtheorem{lemma}{Lemma}
\newtheorem{corollary}{Corollary}
\theoremstyle{definition}

\newtheorem{assumption}{Assumption}
\theoremstyle{remark}
\newtheorem{remark}[theorem]{Remark}
\usepackage[textsize=tiny]{todonotes}

\usepackage[T1]{fontenc}    % fontenc is 
 \usepackage{textcomp}     % required for special glyphs

\usepackage{bm} % math bold 
\usepackage{color,soul}

\usepackage{microtype}
\usepackage[utf8]{inputenc} % allow utf-8 input
\usepackage[T1]{fontenc}    % use 8-bit T1 
\usepackage{url}            % simple URL 
\usepackage{nicefrac}       % compact symbols for 1/2, etc.
\usepackage{microtype,mathrsfs}
\usepackage{graphicx}
\usepackage{booktabs} % for professional tables
\usepackage{float}
\usepackage{amsfonts}       % blackboard math symbols
\usepackage{mathtools,amssymb,comment}
\usepackage{bbm}
\usepackage{breqn} % used for dmath
\usepackage{enumitem}
\usepackage{tcolorbox}
\usepackage{hyperref} % add [pagebackref] to include page number in the bibliography

\definecolor{darkred}{RGB}{230,0,0}
\definecolor{darkgreen}{RGB}{0,150,0}
\definecolor{darkblue}{RGB}{0,0,150}
\hypersetup{colorlinks=true, linkcolor=darkred, citecolor=darkblue, urlcolor=darkblue}
\usepackage{cite}

\newcommand{\sbwc}{self-bounded weak convexity }

\usepackage{mathtools}

\DeclarePairedDelimiterX{\inp}[2]{\langle}{\rangle}{#1, #2}
\newcommand{\poly}{\mathrm{poly}}

\newcommand{\Fh}{\widehat F}

\newcommand{\lamin}[1]{\lambda_{\min}\left(#1\right)}

\newcommand{\negi}{{\neg i}}

\newcommand{\Phixi}{\Phi(w,x_i)}

\makeatletter
\newcommand*{\rom}[1]{\expandafter\@slowromancap\romannumeral #1@}
\makeatother

\newcommand{\mathleft}{\@fleqntrue\@mathmargin0pt}
\newcommand{\mathcenter}{\@fleqnfalse}

\makeatletter
\newcommand{\ssymbol}[1]{^{\@fnsymbol{#1}}}
\makeatother
\newcommand{\hf}{\widehat{F}}

\newcommand{\simiid}{\widesim{\text{\small{iid}}}}

\newcommand{\one}{\mathbf{1}}

\newcommand{\nn}{\notag}

\newcommand{\beq}{\begin{equation}}
\newcommand{\eeq}{\end{equation}}
\newcommand{\bea}{\begin{align}}
\newcommand{\eea}{\end{align}}
\newcommand{\beas}{\begin{align*}}
\newcommand{\eeas}{\end{align*}}

   \newcommand{\footremember}[2]{%
    \footnote{#2}
    \newcounter{#1}
    \setcounter{#1}{\value{footnote}}%
}
  
\newcommand{\widesim}[2][1.5]{
  \mathrel{\overset{#2}{\scalebox{#1}[1]{$\sim$}}}}

\def\bea#1\eea{\begin{align}#1\end{align}}

\usepackage{float}
\title{Sharper Guarantees for Learning Neural Network Classifiers\\ with Gradient Methods}

\author{Hossein Taheri\footremember{ht}{Department of Computer Science and Engineering, University of California, San Diego. Email: \href{mailto:htaheri@ucsd.edu}{htaheri@ucsd.edu}}, \;\; Christos Thrampoulidis\footremember{ct}{Department of Electrical and Computer Engineering, University of British Columbia. Email: \href{mailto:cthrampo@ece.ubc.ca}{cthrampo@ece.ubc.ca}},\;\; and\;\; Arya Mazumdar\footremember{am}{Department of Computer Science and Engineering, University of California, San Diego. Email: \href{mailto:arya@ucsd.edu}{arya@ucsd.edu}}}

\begin{document}

\maketitle

\begin{abstract}
In this paper, we study the data-dependent convergence and generalization behavior of gradient methods for neural networks with smooth activation. Our first result is a novel bound on the excess risk of deep networks trained by the logistic loss, via an alogirthmic stability analysis. Compared to previous works, our results improve upon the shortcomings of the well-established Rademacher complexity-based bounds. Importantly, the bounds we derive in this paper are tighter, hold even for neural networks of small width, do not scale unfavorably with width, are algorithm-dependent, and consequently capture the role of initialization on the sample complexity of gradient descent for deep nets. Specialized to noiseless data separable with margin $\gamma$ by neural tangent kernel (NTK) features of a network of width $\Omega(\poly(\log(n)))$, we show the test-error rate to be $e^{O(L)}/{\gamma^2 n}$, where $n$ is the training set size and $L$ denotes the number of hidden layers. This  is an improvement in the test loss bound compared to previous works while maintaining the poly-logarithmic width conditions. We further investigate excess risk bounds for deep nets trained with noisy data, establishing that under a polynomial condition on the network width, gradient descent can achieve the optimal excess risk. Finally, we show that a large step-size significantly improves upon the NTK regime's results in classifying the XOR distribution. In particular, we show for a one-hidden layer neural network of constant width $m$ with quadratic activation and standard Gaussian initialization that SGD with linear sample complexity and with a large step-size $\eta=m$ reaches the perfect test accuracy after only $\ceil{\log(d)}$ iterations, where $d$ is the data dimension.   
\end{abstract}

\section{Introduction}
\subsection{Overview}
%\subsection{Background and Overview of results}
Neural networks, with their vast capacity for capturing intricate patterns in data, have triggered a paradigm shift in machine learning. Despite the power of these networks in modeling complex relationships, the interplay between their optimization and generalization behaviors (that is the gap between training and test errors) continues to be a compelling area of research. In practice, training neural networks using gradient-based optimization methods often leads to interpolation. That is, deep networks can meticulously fit the training data, driving empirical loss to near-zero and training error to perfect classification. However, these networks also demonstrate the capability to generalize well to unseen data. 
%The duality of such deep network behaviors has ignited a surge of research into their optimization and generalization properties, particularly in the context of overparameterization.
Various recent research endeavors have explored the training and test error guarantees of deep networks, with a focus on the Neural Tangent Kernel (NTK) regime \cite{jacot2018neural,du2019gradient}. One prominent feature of such works is that during gradient descent iterates the network's weights are constrained to move at most a constant distance with respect to overparameterization i.e., $\|w^\star - w_0\| = O_m(1)$, where $w^\star\in\R^p$ denotes the vector of target weights, $w_0$ is the initial weight vector, and $m$ is the network width \cite{chen2020much,ji2019polylogarithmic,telgarsky2023feature}. 

Yet, even for the relatively simple setting of learning deep nets in the kernel regime, the existing generalization bounds are still suboptimal. Moreover, the boundaries of the kernel regime are still largely unknown and an active area of research \cite{liu2022loss,banerjee2022restricted,telgarsky2023feature}. While the kernel regime can partially demonstrate the behavior of neural networks, the resulting guarantees often require large width, small step-size or large iteration and sample complexities. There is increasing evidence in recent years that for certain class of data distributions neural networks can overcome these limitations by using a large step-size which allows the network's parameters to move a long distance from initialization, often leading to better sample and computational complexities \cite{damian2022neural,ba2022high}. 

In this work, we study the generalization and convergence behavior of gradient-based algorithms in neural nets with smooth activation functions for a wide class of data distributions. Our first result characterizes the test and train loss rates for classification problems under the condition that deviation from initialization is bounded depending on the network's width. In particular, for $L$-hidden layer networks our results hold under $\|w^\star - w_0\| \lesssim m^{O(1/L)}$, allowing the network's weights to move from initialization a distance increasing with $m$. This shows that the kernel regime continues to hold for a wider range of setups than previous results for which the deviations are restricted to be constant in $m$. The key reason for this improvement is exploiting the objective's Hessian structure in the gradient-descent path. More importantly, using the Hessian information enables us to develop, for the first time, \emph{algorithm-dependent} generalization bounds of deep neural networks. As will be discussed throughout the paper, the bounds we derive via algorithmic-stability are tighter than previous relevant bounds in the literature.
\begin{table}[t]
\begin{tabular}{p{4.2cm} p{1.7cm} p{2.5cm} p{2cm} p{3.4cm}}
\hline
& {\bf Activation} & {\bf \; Width} & {\bf Train Loss}  &{\bf Test Loss}\\
\hline\\[-0.14in]
\hspace{-0.02in}\cite{chen2020much,bartlett2017spectrally}& ReLU &$ \hspace{-0.02in}\Omega(\poly(\frac{\log(n)}{\gamma}))$ & $\;\widetilde O(\frac{1}{\gamma^2 T})$ & $\widetilde O(\frac{e^{O(L)}}{\gamma^2} \sqrt{\frac{m}{n}}%\wedge(\frac{L^{3/2}}{\gamma^2 \sqrt{n}} + \frac{L^{11/3}}{\gamma^2 m^{1/6}}) 
+ \frac{1}{\gamma^2 T})$ \\[.2in]
Thm. \ref{thm:train-test}, Cor. \ref{cor:NTK} & Smooth & $\Omega(\poly(\frac{\log(n)}{\gamma}))$   & \;$\widetilde O(\frac{1}{\gamma^2 T})$ & $\widetilde O(\frac{e^{O(L)}}{\gamma^2 n} + \frac{1}{\gamma^2 T})$ \\[.03in]
%\hspace{-.2in}Thm \ref{thm:generalization} & $\Omega(n^3)$   & \,\, -- & $F(w^\star)+\frac{\|w^\star-w_0\|^2}{\sqrt{n}}$
%\vspace{.05in}\\
\hline
\end{tabular}
\caption{Comparison of our results on learning deep nets with GD under NTK separability condition to related prior results. Here   $m:$ network width, $L:$ network depth, $\gamma:$ NTK-margin, $n:$ number of samples and $T:$ number of iterations. %Our test-error bound is independent of the network width and is tighter than the bound in      \cite{chen2020much} which is derived by Rademacher complexity bounds of deep nets \cite{bartlett2017spectrally,cao2019generalization}.}
}
\label{tab:lit}
\end{table}
\begin{table}[t]
\begin{tabular}{p{6cm} p{2cm} p{3cm} p{2.2cm} p{1cm}}
\hline
&{\bf Activation} & {\bf Width} &{\bf Iteration } &{\bf Sample}\\
\hline\\[-.1in]
\hspace{-0.02in}\cite{telgarsky2023feature,chen2020much,taheri2024generalization}, Cor. \ref{cor:NTK}   & all & $\Omega(\poly(d))$ & $d^2$ & $\widetilde O(d^2)$
\\[0.24in]
\hspace{-0.02in}\cite{glasgowsgd} & ReLU & $\Omega(\poly(\log(d)))$ & $\poly(\log(d))$ & $\widetilde O(d)$\\[0.05in]
Thm. \ref{thm:xor} & Quadratic & $\Omega(1)$ & $\log(d)$ & $\widetilde O(d)$\\[.03in]
\hline
\end{tabular}
\caption{Comparison of our findings on learning the $d$-dimensional XOR distribution with SGD to relevant prior results.}
\label{tab:xor}
\end{table}
We specialize these results to a well-known NTK separability condition tailored for noiseless data and show that our results substantially improve the prior results on the test error performance while still allowing the width to be small, specifically poly-logarithmic on sample size (Table \ref{tab:lit}). A more detailed comparison is deferred to Section \ref{sec:lit}. We also consider the case of noisy data distributions and show that deep models are consistent, i.e., they can achieve the optimal test loss in the presence of noise as the sample size grows. 
%Numerical experiments on the train loss and generalization loss behavior of GD in the interpolating regime for FashionMNIST and MNIST datasets with multi-layer neural networks are presented in Section \ref{sec:exp}.
\par 
While these results improve upon the existing bounds within the NTK regime, we show in Section \ref{sec:xor} that using a large step-size can further improve both the computational and sample complexities. In particular, we show for the stylized setup of data following the XOR distribution, a two-layer neural network with quadratic activation reaches zero test error after only $\log(d)$ steps of SGD with an aggressive step-size $\eta=m$. A comparison of our findings with the guarantees of the kernel regime with both GD \cite{taheri2024generalization} and one-pass SGD \cite{telgarsky2023feature,chen2020much} and the most relevant work in the feature learning regime \cite{glasgowsgd} is summarized in Table \ref{tab:xor}. 

Below is a summary of our contributions.
\begin{itemize}[leftmargin=*]
    \item In Theorem \ref{thm:train-test}, we develop sufficient conditions for the global convergence of gradient descent in deep and smooth networks and show that if $m=\Omega(\|w^\star-w_0\|^{6L+4})$, the training loss is bounded by $O(\frac{\|w^\star-w_0\|^2}{\eta T})$, where $\eta$ is the step-size and $w^\star$ can be any choice of network weights that achieves small training loss. Under similar conditions on $m$, we show the generalization error is bounded by $O(\frac{\|w^\star-w_0\|^2 G_0^2}{n})$ where $G_0$ is the Lipschitz parameter of the network at initialization. 
    %As this bound is derived via an algorithmic stability argument, we are also able to derive a tighter version of this bound by showing that generalization gap is obtained as the averaged training performance 
    %$\frac{\eta G_0^2  }{n} \E_{\mathcal{S}}\left[ \sum_{t=0}^{T-1}\Fh(w_t)\right]$.
    %The results in Theorems \ref{thm:train-test}-\ref{thm:train-test} reveal the close connection between training and test errors in deep learning. These results generally hold without the NTK assumptions 
    %Unlike previous works in literature the test loss captures the role of initialization and the distance traveled by network weights in \emph{Frobenius} norm.  
    \item In Corollary \ref{cor:NTK}, we interpret these results by specializing them to a commonly-used margin-based NTK separability condition. The results of the corollary and comparison to previous works in literature are summarized in Table \ref{tab:lit}. A promising feature of our approach is that the test loss bound does not have an unfavorable dependence on the width while still maintaining minimum poly-logarithmic width conditions, which is new in the context of deep learning. To the best of our knowledge this is the tightest test error bound for deep nets trained by GD in the NTK regime.
    \item We consider the more general case of noisy data with non-vanishing optimal test loss in Theorem \ref{thm:generalization} and show that under a polynomial growth condition on network width, GD achieves a convergence rate of $1/\sqrt{n}$ to the optimal loss after $T = \sqrt{n}$ iterations. 
    \item In Section \ref{sec:xor}, we consider the $d$-dimensional XOR distribution and show that a one-hidden layer network of constant width after exactly $\log(d)$ iterations of SGD with step-size $\eta=m$ achieves perfect test accuracy with $n=\widetilde O(d)$ samples, considerably improving kernel regime's limitations.  
\end{itemize}
\subsection{Prior works}\label{sec:lit}
\paragraph{Generalization of deep nets.} Among prior works on the generalization capabilities of deep networks, the only initialization dependent bounds were provided in 
\cite{bartlett2017spectrally} obtaining bounds of order $O(\frac{\mathcal{R}}{n})$ where the Rademacher complexity is derived as $\mathcal{R}:=(\prod_{i=1}^L \|W_i\|_{2})(\sum_{i=1}^L \frac{\|W_i^{\top}-M_i^{\top}\|_{2,1}^{2 / 3}}{\|W_i\|_{2}^{2 / 3}})^{3 / 2}.$ Here $\|W_i\|_{2}$ is the spectral norm of the weight matrix of layer $i$ (typically a constant) and $M_i$s are any data-independent matrices. Thus one can choose $M_i=W_{i,0}$, i.e., the initialization weight matrix. In fact, the above bound resembles the bound that we obtained via an optimization-dependent stability analysis. However, note that $\mathcal{R}$ depends on the distance traversed by weights through the $\ell_{2,1}$ norm  which is always larger than the Frobenius norm, and in the worst case, the gap can be significantly large depending on the width. To see this, note that for a matrix $V\in\R^{m\times m}$ it holds that $\|V\|_{2,1} \le \sqrt{m}\|V\|_{F}$. We note that ``initialization-independent'' bounds (e.g., \cite{neyshabur2018pac,golowich2018size}) that are usually proportional to $\|w_t\|$  (rather than $\|w_t-w_0\|$) are strictly looser than the bound we obtain. This is primarily due to the fact that $\|w_t-w_0\|$ can be much smaller than $\|w_t\|$ and in fact as our experiments show $\|w_t-w_0\|$ is of constant order and can even decrease with width. Whereas, $\|w_t\|$ (or $\|w_t\|/\sqrt{m}$ due to the normalization in our setup) grows by increasing $m$, making the initialization-independent bounds potentially grow with width at the rate $O(\sqrt{m})$, despite lacking an explicit dependence on $m$. Hence, for wide networks, prior generalization bounds of deep neural nets based on Rademacher complexity are larger than the bound we derive in Theorem \ref{thm:train-test}.
\par 
\paragraph{Test rates under the NTK separability condition.} Other works that provide generalization bounds and optimization guarantees for neural nets include \cite{cao2019generalization,nitanda2019gradient,ji2019polylogarithmic,chen2020much,richards2021learning,taheri2024generalization,wang2023generalization}. In particular, \cite{ji2019polylogarithmic} derived the width condition $m=\Omega(\frac{\poly(\log(n))}{\gamma^8})$ for achieving the $\frac{1}{\gamma^2\sqrt{n}}$-test error rate in two-layer nets via a uniform-convergence argument \cite{shalev2014understanding}. 
%It is worth noting that the generalization bounds in \cite{ji2019polylogarithmic} can be improved to $\tilde O(\frac{1}{\gamma^2 n})$ by a smoothness-based Rademacher complexity guarantee from \cite{srebro2010smoothness}. However, this improvement comes at the cost of introducing additional poly-logarithmic factors and very large constants of order $10^5$ into the generalization bound which can make such bounds potentially vacuous. 
This bound was extended to deep networks in \cite{chen2020much} with a generalization gap of order 
$$\widetilde O\left(\frac{4^L}{\gamma^2} \sqrt{\frac{m}{n}}\wedge\left(\frac{L^{3/2}}{\gamma^2 \sqrt{n}} + \frac{L^{11/3}}{\gamma^2 m^{1/6}}\right)\right),$$
where $\wedge$ takes the minimum of two quantities. As discussed earlier, this bound is dependent on width since it is derived essentially by taking the minimum of two generalization bounds based on Rademacher complexity derived in \cite{bartlett2017spectrally} and \cite{cao2019generalization}. Importantly, in the small width regime the bound simplifies into $\tilde O(\frac{4^L}{\gamma^2}\sqrt{\frac{m}{n}})$, which has an undesirable dependence on the width. %Moreover, the analysis in \cite{chen2020much} is done under the NTK approximation where $\Phi(w,\cdot)=\Phi(w_0,\cdot) + \langle \nabla _w\Phi(w_0,\cdot),w-w_0\rangle,$ and thus it is unclear when this approximation is valid. 
In this paper, we improve the generalization gap to $\tilde O(\frac{e^{O(L)}}{\gamma^2 n})$ under the width condition $m=\Omega(\poly(\log(n)/\gamma))$. To the best of our knowledge, these are the smallest generalization bound and width condition in literature up to now for learning deep neural nets. 
%Importantly, our approach is exact and does not rely on the NTK approximation. 
The key reason behind the improvement is leveraging the Hessian structure of the objective throughout the gradient descent iterates. The improved generalization guarantees result from the algorithmic dependency of our bounds and in fact the bounds can even be expressed such that they solely depend on the cumulative training loss (c.f. Eq. \ref{eq:thm2genopt}). As the training loss captures the role of initialization and is independent of width, the resulting generalization bounds share the same favorable properties.  
%Finally, with regards to Rademacher-complexity of two-layer networks, we note that 
\par

\paragraph{Feature learning and the XOR distribution.} Some recent works have pointed out the limitations of the kernel regime in understanding the full power of neural networks \cite{abbe2022merged}. In particular, as in the kernel regime, the networks weights are bounded not to move significantly from initialization, the learned features are not considerably different from those learned at initialization.  On the other hand, it is hypothesized that neural networks can learn the true underlying features of the data distribution if the network weights are allowed (by large step-sizes or avoiding early-stopping) to move a large distance from initialization. This phenomenon was first proved for specific regression tasks where the labels essentially only depend on a small number of features, such as when $y=g(Ux)$ for $U\in\R^{k\times d}$ where $k\ll d$ \cite{damian2022neural,ba2022high,abbe2022merged,cui2024asymptotics}, by one large SGD step leading to superior sample complexities compared to the kernel regime. For classification tasks, some focus has been on the XOR distribution(a.k.a. parities) \cite{wei2019regularization}. Recent works have studied the problem of learning the $d$-dimensional XOR distribution using neural networks in both NTK and feature learning settings \cite{barak2022hidden,telgarsky2023feature,taheri2024generalization,glasgowsgd}. Specifically, it has been shown that under NTK with a sufficiently small step size, a polynomially wide network requires $d^2$ GD steps and $d^2$ sample size. Some studies have achieved linear sample complexity for learning XOR \cite{bai2019beyond,glasgowsgd,telgarsky2023feature}; but these methods involve more computational effort compared to our results. The work most related to ours is  \cite{glasgowsgd}, which demonstrated that with a particular Gaussian initialization, a ReLU network requires $\poly(\log(d))$ large SGD steps and $\poly(\log(d))$ neurons to learn XOR with linear sample complexity. In contrast, we show that by using a quadratic activation, learning this distribution requires only $\log(d)$ large steps with a constant-width network, while maintaining the same linear sample complexity.

%%%%%%%%%%%%%%%%%%%%%% NOTATION %%%%%%%%%%%%%%%%%%%%%%%%%%%%%%%
\subsection*{Notation}
Probability and expectation with respect to the randomness in random variable $x$ are denoted by $\Pr_x(\cdot)$ and $\E_x[\cdot]$. We use the standard complexity notation $\lesssim,o(\cdot),O(\cdot
),\Theta(\cdot),\Omega(\cdot
)$ and use $\tilde{o}(\cdot), \tilde{O}(\cdot
),\tilde\Theta(\cdot),\tilde\Omega(\cdot
)$ to hide polylogarithmic factors. We denote $a\wedge b:= \min\{a,b\}$. The Gradient and Hessian of the model $\Phi:\R^{p\times d}\rightarrow \R$ with respect to the first input (i.e., weights) are denoted by $\nabla \Phi$ and $\nabla^2 \Phi$, respectively. The minimum eigenvalue of a symmetric matrix is denoted by $\lambda_\min(\cdot)$. We use $\|\cdot\|$ for the Euclidean norm of vectors and $\|\cdot\|_2$ for the spectral norm of matrices. We denote $[w_1,w_2]:=\{w\,:\,w=\alpha w_1+(1-\alpha) w_2, \alpha\in[0,1]\}$ the line between $w_1,w_2\in\R^{p}$. 

\section{Main results}\label{sec:main}

%\subsection{Preliminaries}
Throughout the paper, we consider the following unregularized objective for a neural network classifier parameterized with $w\in\R^p,$
\bea\label{eq:logregDNN}
\min_{w\in\R^{p}} \Fh(w):=\frac{1}{n}\sum_{i=1}^{n} f\left(y_i \Phi(w,x_i)\right),
\eea
with data points satisfying $\|x\|\le 1$, the binary labels $y_i\in\{\pm 1\}$ and $f(\cdot)$ is a loss function for classification tasks such as the logistic loss, $f(t):= \log(1+e^{-t})$ and $\Phi(\cdot,x)$ is the network's output. We also define the test loss as $F(w):=\E_{x,y}[f(y\Phi(w,x))].$
\subsection{Train and test loss bounds in deep nets}
In our first theorem, we establish conditions for the width and target weights of the network that guarantees the training loss decays to zero if the network can interpolate the training set. We consider gradient descent update rule where at any iteration $t\le T$: $w_{t+1}=w_t-\eta\nabla \hf(w_t).$
Before stating the theorem, we note that this result is valid under the standard descent-lemma condition for the step-size, as stated in Lemma \ref{lem:des} in the appendix. In particular, the descent lemma holds if the step-size satisfies the standard Eq. \ref{eq:descon} in the appendix.   
\begin{theorem}[Train \& Test loss of deep nets]\label{thm:train-test} 
Consider the $L$-layer neural network with width $m$ where $\Phi(w,x):= \frac{1}{\sqrt m}  W_{L+1}^\top(\frac{1}{\sqrt m} \sigma(W_{L}^\top \cdots \frac{1}{\sqrt{m}}\sigma (W_1^\top x)\cdots)$ and $\sigma$ is a $1$-smooth and $1$-Lipschitz activation function such that $\sigma(0)=0$. Moreover, let $\beta_L$ be a constant that only depends on $L$, let all parameters of the network be initialized as i.i.d. standard Gaussian and assume the step-size satisfies the condition of the descent lemma. Fix $T$ and assume the target weights vector $w^\star\in\R^p$ that obtain small training loss such that 
\bea\label{eq:firstcon}
\rho^\star\geq \max\left\{\sqrt{\eta T \widehat F(w^\star)},\sqrt{\eta \widehat F(w_0)}\right\}.
\eea
where $\rho^\star:=\|w^\star-w_0\|$. Moreover, assume the width $m$ is large enough such that it satisfies,
\bea\label{eq:secondcon}
m \ge 4\beta_L^2 (6 {\rho^\star})^{6L+4}
\eea
Then, $\|w_t-w_0\|=O(\rho^\star)$ and the training loss satisfies with high probability over initialization,
\bea\label{eq:train_rate}
\Fh(w_T) \le\frac{4{\rho^\star}^2}{\eta T}.
\eea
Moreover, assume for every $n$ samples from the data distribution there exists $w^\star$ satisfying Eqs. \ref{eq:firstcon}-\ref{eq:secondcon}. Then, the expected generalization gap satisfies with high probability over the randomness of initialization,
\bea\label{eq:thm2genopt}
\E_{\mathcal{S}}\left[{F}(w_T)-\widehat F(w_T)\right]\le 2.2\,\frac{\eta(G_0+1/4)^2  }{n} \,\E_{\mathcal{S}}\left[ \sum_{t=0}^{T-1}\Fh(w_t)\right],
\eea
where the expectation is over the randomness in the training set denoted by $\mathcal{S}$ and $G_0$ is the Lipschitz parameter of network at initialization  i.e., $\|\nabla \Phi(w_0,\cdot)\|\le G_0$.
\end{theorem}
In words, the main condition of the theorem is the existence of network weights denoted by $w^\star$ that achieves small training error (Eq. \ref{eq:firstcon}) and its distance from initialization is at most $O(m^{1/(6L+4))})$ as implied by Eq. \ref{eq:secondcon}. Under these conditions, the training loss is controlled solely by $\|w^\star-w_0\|$ and has no explicit dependence on the width or depth of the network. 
%The new result here, is the establishment of width condition $m \gtrsim {\|w^\star-w_0\|}^{(6L+4)}$.
As it will be stated in Corollary \ref{cor:NTK}, in the NTK regime with margin $\gamma$ it holds $\|w^\star-w_0\|=O(\log(n)/\gamma)$, leading to the width condition $m=\Omega(\poly(\log(n)/\gamma))$. This is unlike previous results which either required polynomial width (such as \cite{liu2022loss,cao2019generalization}) or led to sub-optimal bounds (e.g., \cite{chen2020much,ji2019polylogarithmic}).

In general, the theorem is valid for any feasible minimizer $w^\star\in\R^p$. Thus, we can choose $w^\star$ with smallest value for $\|w^\star-w_0\|$ to optimize the bounds. With such choice of $w^\star$, the distance the weights obtained by GD travel is also minimized as $\forall t\in[T] : \|w_t-w_0\| = O(\|w^\star-w_0\|)$. Thus, gradient descent tends toward solutions which attain small loss and lie at minimum possible distance from initialization. This is in line with related prior observations in several other works such as \cite{du2019gradient,oymak2019overparameterized}. 

%We highlight that the depth of the network does not explicitly appear in the training guarantees and with fixed step-size the training loss only relies on the distance the weights travel from initialization to target. 
%As a remark, we expect the bound is tight up to numerical constants and can not be improved in general since the bound in Eq. \ref{eq:train_rate} is analogous(up to constants) to the well-known convergence guarantees for smooth and convex objectives.

%Our next result, presents the algorithmic-dependent generalization and test loss guarantees at any iteration of GD, under almost the same conditions as in Theorem \ref{thm:train-test}. 
%\begin{theorem}[Test loss]\label{thm:test}
%Let Assumptions of Theorem \ref{thm:train-test} hold. Fix $T$ and assume the target weights vector $w^\star\in\R^p$ that achieves small loss on any $n$ samples from the data distribution such that 
%\bea\nn
%\rho^\star\geq \max\left\{\sqrt{\eta T \widehat F(w^\star)},\sqrt{\eta \widehat F(w_0)}\right\}.
%\eea
%where $\rho^\star:=\|w^\star-w_0\|$. Moreover, assume the width $m$ is large enough such that,
%\bea\label{eq:widthcon}
%m \ge 4\beta_L^2 (6 {\rho^\star})^{6L+4}
%\eea

The theorem also shows the sample complexity and iteration complexity of learning deep networks with gradient descent and demonstrates the role of initialization and weight's norms on the test error. Note that by replacing our training loss guarantees, the generalization gap simplifies into:
\bea\label{eq:gensubopt}
\E_{\mathcal{S}}\left[{F}(w_T)-\widehat F(w_T)\right] \le 9\frac{{\rho^\star}^2 (G_0+1/4)^2  }{n}\,  .
\eea 
% In particular, the test loss at iteration $T$ is bounded as,
% \bea
% \E_{\mathcal{S}}\left[{F}(w_T)\right] \le 9\,\frac{{ \rho^\star}^2(G_0+1/4)^2  }{n} + 4\,\frac{{\rho^\star}^2}{\eta T}.
% \eea
Hence, the test loss after $T=\Theta(n)$ iterations takes the form of 
\bea\label{eq:tstbund}
\E_{\mathcal{S}}\Big[{F}(w_T)\Big] =O\left(\frac{\Big\|w^\star-w_0\Big\|^2 G_0^2}{n} + \frac{\Big\|w^\star-w_0\Big\|^2}{\eta n}\right),
\eea 
where the first term is the generalization error and the second term is the training loss. Remarkably, Eq. \ref{eq:tstbund} shows the tight correlation between the two terms, as the generalization gap is virtually the optimization error scaled by the squared Lipschitz constant $G_0^2$. This is indeed the consequence of Eq. \ref{eq:thm2genopt} which bounds the generalization gap based on the cumulative optimization error. 
\par
At a high-level, the test loss essentially depends on two quantities: (i) the Euclidean distance between the target weights and the initialization and (ii) the Lipschitz parameter of network at initialization. Due to the algorithmic-dependent nature of our generalization bounds and unlike the prior bounds in literature, the bound in Eq. \ref{eq:tstbund} captures the role of initialization on the test error: smaller deviations from initialization lead to smaller test error bounds and in particular the bound approaches zero as $\rho^\star$ goes to zero. As discussed earlier, gradient descent favors such solutions with small deviations from initialization.
In addition to  ``distance from initialization'', the test error bounds also depend on the squared Lipschitz parameter of the network. For standard Gaussian initialization, it can be shown (cf. Section \ref{sec:gradnorm}) that $G_0\lesssim e^{O(L)}$ which introduces an exponential dependence on depth to the generalization bound. We remark that this dependence also appears in the corresponding bounds derived via uniform convergence and Rademacher complexity (e.g., \cite{bartlett2017spectrally,golowich2018size,chen2020much}) through the term $\prod_{i=1}^L\|W_i\|_2$. %As for our initialization $\|W_i\|_2=\Theta(1)$, the resulting generalization bounds will have an explicit exponential dependence on $L$.    

An interesting feature of our approach is the algorithmic dependent bound in Eq. \ref{eq:thm2genopt}. This bound is generally tighter than the bound in Eq. \ref{eq:gensubopt}. With the descent lemma condition on the step-size (c.f. Lemma \ref{lem:des}) it holds that $\eta<1/(G_0^2+1/4)$ which simplifies the bound into:
\bea\label{eq:stabsim}
\E_{\mathcal{S}}\left[{F}(w_T)-\widehat F(w_T)\right] \le \frac{2.2}{n} \,\E_{\mathcal{S}}\left[ \sum_{t=0}^{T-1}\Fh(w_t)\right].
\eea
Hence, we have a bound which only depends on the training performance and the number of training samples. In our experiments in Section \ref{sec:exp}, we compute this bound for real-world data and compare it with the empirical results for generalization and test loss.  
\par

\begin{remark}
Our analysis relies on the recent progress in characterising the spectral norm of the deep net's Hessian during GD updates \cite{liu2020linearity,banerjee2022restricted}. In particular, \cite{liu2020linearity} proved that with standard Gaussian initialization, the model's Hessian is bounded with high probability by $\|\nabla^2 \Phi(w,x)\| = O(\frac{R^{3L}}{\sqrt{m}})$ if $\|w-w_0\|\le R$. These guarantees can also be used to study the convergence rate of deep networks trained by quadratic loss as done by \cite{liu2020linearity,liu2022loss} but their approach leads to excessively large width conditions. In contrast, here we consider classification tasks with an improved poly-logarithmic width requirement and also study the generalization performance.

Recently, the algorithmic stability has been employed for two-layer neural nets in \cite{taheri2024generalization,richards2021learning}which is an improved analysis of the stability-based approach typically used for convex objectives in \cite{bousquet2002stability,hardt2016train,lei2020fine}. Here, we essentially extend the stability analysis to deep networks. Compared to the two-layer nets, here in every iterate of gradient descent, the Hessian's norm guarantees depend on the network's weights. In particular, the analysis has to take into account that both $\|w_t-w_0\|$ and $\|w_t-w^\star\|$ remain small during GD updates. This is necessary in order to ensure the Hessian's norm guarantees and the approximate quasi convexity property hold during all GD iterates. We do this by an induction based argument which bounds these terms based on the fixed quantity $\|w^\star-w_0\|$, which is guaranteed to be bounded based on width by assumption. 
%A more recent analysis by \cite{banerjee2022restricted} derived similar bounds but with better dependence on the depth. However, their initialization is different than the commonly-used standard Gaussian initialization that we also consider here. Additionally, \cite{liu2020linearity} also studied training loss guarantees of deep smooth nets trained by strongly convex loss functions such as the quadratic loss. 
   
\end{remark}

\subsubsection{Specializing to the NTK-separablity condition} The results in the last section can be specifically applied to a class of data distributions that includes the XOR distribution which will be discussed more through the rest of the paper. Before stating our result in the corollary, we state the neural tangent kernel (NTK) separability assumption.
\begin{assumption}[NTK-separability \cite{nitanda2019gradient}]\label{ass:ntkdata}
Assume the tangent kernel of the model at initialization separates the data with margin $\gamma$ i.e., for a unit-norm vector $w\in\R^p$ it holds for all $i \in [n]:$
$y_i\left\langle\nabla \Phi\left(w_0, x_i\right), w\right\rangle \geq \gamma$.
\end{assumption}
In words, the above assumption implies that the features learned by the gradient at initialization can be linearly separated by some weights $w$. This assumption is commonly used in deep learning theory literature for studying classification tasks \cite{chen2020much,ji2019polylogarithmic}.
\begin{corollary}[NTK results]\label{cor:NTK}
Consider the same setup as Theorem \ref{thm:train-test}, let the loss function be logistic loss and let Assumption \ref{ass:ntkdata} hold. Define constant $B>0$ that bounds the model's output at initialization i.e., $\forall i\in[n]: |\Phi(w_0,x_i)| < B$. Assume the width is large enough such that 
$m\ge \beta_L^2 \left(\frac{2B+\log(1/\eps)}{\gamma}\right)^{6L+4}.$
Then there exists $w^\star\in\R^p$ such that $F(w^\star)\le \eps$ and $w^\star$ lies at a bounded Euclidean distance from initialization such that $\|w^\star-w_0\|\le \frac{1}{\gamma}(2B+\log(1/\eps))$.
\end{corollary}

To interpret this result, we apply this result to Theorem \ref{thm:train-test} and fix $\eps=1/T$. First note that the output of network at initialization is constant with high-probability over initialization (e.g., see \cite[Lemma F.4]{liu2020linearity}, \cite[Lemma 4.4]{cao2019generalization}) which implies $B=\tilde O(1)$. Overall, with the stopping condition $T=\Theta(n)$ and recalling that $G_0\le e^{O(L)}$, Corollary \ref{cor:NTK} yields the expected test error rate of order $$\widetilde O\left(\frac{e^{O(L)} }{\gamma^2 n} + \frac{1}{\gamma^2 T}\right),$$
under the condition that $m= \Omega(\poly(\log(n)) \cdot\beta_L/\gamma^{6L+4})$. We remark the bound does not have any explicit dependence on $m$. This in itself is not surprising as one intuitively expects the bound not to scale unfavorably with width. Although recent works have derived width independent generalization bounds for two-layer networks \cite{ji2019polylogarithmic,telgarsky2023feature,taheri2024generalization}, we are not aware of any prior work proving width-independent bounds for learning multi-layer networks with GD. As discussed earlier in the introduction, the bound derived in the closely related work \cite{chen2020much} scales unfavorably with $m$ as it grows at the rate $\sqrt{m/n}$. In fact, the authors refer to deriving width-independent bounds as an open problem in \cite[Sec 3.1]{chen2020much}.

%indeed \cite{chen2020much} mention deriving width independent bounds as an open problem.

\subsubsection{Consistency of GD with noisy data}

The results of the last section mainly apply to data settings when the network can find the optimal solution within a bounded  distance from initialization that depends on the network's width. This setting was specially tailored to the noiseless case where achieving vanishing test loss was possible. We discuss next the case of learning deep nets by noisy data and show that achieving optimal test loss might be feasible in this setting.

\begin{theorem}[Test error for noisy data]\label{thm:generalization}
Consider the same setup and notation as Theorem \ref{thm:train-test} and assume the width of the network satisfies $m\ge\beta_L^2 n^{3L+3}$ and the step-size satisfies the conditions of the descent lemma. Then, with high probability over initialization the expected test loss at iteration $T=\sqrt{n}$ is bounded as:
\bea\label{eq:thm3boundmain}
\E_{\mathcal{S}}\Big[F(w_T)\Big] \le\left( 1+ \frac{4}{\sqrt{n}}\right) \left( F(w^\star) + \frac{{\rho^\star}^2}{\eta \sqrt{n}} + \frac{{\rho^\star}^2}{n \sqrt{n}}  + \frac{1}{\sqrt{n}}\right),
\eea
where $\rho^\star:=\|w^\star-w_0\|$ and $w^\star$ is the minimizer of the population loss i.e., $w^\star = \arg\min_{w\in\R^p} F(w)$. 
\end{theorem}
The result above shows that GD reaches optimal level given that $\sqrt{n}\gg\|w^\star-w_0\|^2$; as for large $n$, it leads to the simplified expression 
$$\E_{\mathcal{S}}\Big[F(w_T)\Big] - F(w^\star)= O\left(\frac{\Big\|w^\star-w_0\Big\|^2 +F(w^\star)}{\sqrt{n}}\right).$$
Thus, even in non-interpolating regime, GD can still achieve the optimal solution of over-parameterized deep networks.  

The bound in Eq. \ref{eq:thm3boundmain} is derived via a stability argument which bounds the generalization gap based on the training performance. However, contrary to the conditions of Theorem \ref{thm:train-test} here we do not have the interpolation condition as $F
(w^\star)$ is not vanishing.  This comes at the expense of a larger width condition where the width is polynomial in $n$ whereas poly-logarithmic width was sufficient in Theorem \ref{thm:train-test}. The condition on early stopping further guarantees that the test loss reach the Bayes error given $n$ is sufficiently large. It is worth noting that the setting above still is operating almost in the NTK regime. This can be verified by the observation that for the bound to be meaningful it should hold that $\|w^\star-w_0\| \ll \sqrt{n}\lesssim m^{O(1/L)}$, as per the theorem's width condition. Finally, we remark that the result of Thm \ref{thm:generalization} nicely connects to recent empirical and theoretical results \cite{li2020gradient} which show that with early-stopping, GD can find a good solution for clustered data with label noise.

\subsection{Overcoming NTK limitations for XOR dataset}\label{sec:xor}
Next, we show the previous NTK guarantees can be improved by using large step-sizes. We consider the stylized set-up of learning the $d-$dimensional XOR data distribution. Consider one-hidden layer network with quadratic activation where $x\in\R^d,w_i\in \R^d$:
\begin{align*}
    \Phi(w,x) = \frac{1}{2m}\sum_{i=1}^m a_i (x^\top w_i)^2.
\end{align*}
In the above, $a_i \in \{\pm 1\}$ are fixed during training and satisfy $\sum_i a_i=0$ and only the first layer weights are trained. For initialization of first layer's weights we have $\forall i\le m,j\le d: w_{0,ij}\simiid N(0,\frac{1}{d})$. The data points $(x,y)\in\{\pm 1\}^d\times\{\pm 1\}$ are uniformly drawn from the resulting $d$-dimensional distribution of $2^d$ points where the labels are determined as $y=x(1)\cdot x(2).$ For the result in the next theorem, we assume the loss function is the linear loss $f(t) = -t$ and consider mini-batch SGD with the update rule $w^{t+1} = w^t - \eta \widehat{F}(w^t)$, where at each step, $n$ data points are drawn i.i.d. from the distribution to form $\widehat{F}(\cdot)$.
The next theorem shows the computational and sample complexities of learning this dataset for the aforementioned setup.
\begin{theorem}[Improved guarantees for learning XOR]\label{thm:xor}
Assume mini-batch SGD with batch-size $n\ge d\cdot\log^{14}(d)$ and step-size $\eta = m$. Then, after $T=\log(d)$ iterations the test accuracy satisfies with probability at least $1-e^{\log(m)-\log^2(d)}-e^{-\frac{m}{16}}-o_d(1)$ over the randomness of initialization and data sampling,
\begin{align}\label{eq:texor}
\E_{x,y} \Big[\one_{\{y\Phi(w^T,x)>0\}}\Big]\ge 1-o_d(1).
\end{align}
\end{theorem}
Hence, with logarithmic number of SGD steps we use $\widetilde O(d)$ samples in total to reach almost perfect test accuracy. We remark based on the guarantees of Theorem \ref{thm:xor}, the network's width can be constant and at most must be polynomial in $d$ in order for Eq. \ref{eq:texor} to hold with high probability. This aligns with our experiments in Fig. \ref{fig:4}, demonstrating that the network's width can be independent of data dimension as a network of constant width (where $m=20$) suffices for learning arbitrary high-dimensional XOR data. We also note the considerable gains in iteration, sample and computational complexities by Theorem \ref{thm:xor} resulting from escaping the NTK regime with our step-size selection. Contrary to Theorem \ref{thm:train-test} which required small step-size in accordance with the descent lemma, here the step-size is proportional to width. For a comparison with recent results for this dataset we refer to Table \ref{tab:xor}. To the best of our knowledge these are the best complexities on iteration and network width for this setup. 
%It remains open to explore whether the logarithmic iteration complexity can be further improved. 

\begin{remark}
The proof of Theorem \ref{thm:xor} (provided in Appendix \ref{sec:proof_xor}) is nuanced and involves computing expected weights and their corresponding error terms due to SGD sampling noise for each parameter of the network. It is then showed that signal strength (i.e., the strength of important features) grows as $\frac{2^t}{\sqrt{d}}$ whereas the error terms due to sampling noise and initialization grow at most at the rate $(1+\frac{1}{\poly(\log(d))})^t \frac{\sqrt{d}}{\poly(\log(d))} + \frac{2^{t}t}{\sqrt{d}\cdot\poly(\log(d))}$. Therefore, after $\log(d)$ steps the noise strength reaches $\frac{\sqrt{d}}{\poly(\log(d))}$ whereas the signal's magnitude is at least $\sqrt{d}$, letting the signal outgrow the noise and leading to the network classifying every point correctly. 
%Finally, we remark that the choice of step-size is made for simplicity of presenting the final results. Concretely, the step-size is allowed to change proportionally with $m$ i.e., $\eta=\Theta(m)$. 
\end{remark}
\section{Numerical Results}\label{sec:exp}
\paragraph{Experiments on learning under NTK with small step-size.}
In this section, we present numerical results on the behavior of the generalization bound derived in Theorem \ref{thm:train-test} for real-world data (FashionMNIST and MNIST datasets) and compare it with the empirical generalization gap.  
\par
For demonstrating our theoretical results, we are interested in the algorithmic-based generalization bound (Eq. \ref{eq:stabsim}) derived as,
\bea\label{eq:stabsim2}
\E_{\mathcal{S}}\Big[{F}(w_T)&-\widehat F(w_T)\Big] \le \frac{2.2}{n} \,\E_{\mathcal{S}}\left[ \sum_{t=0}^{T-1}\Fh(w_t)\right].
\eea
We find it helpful to note that the bound above requires the width condition in the theorem (i.e., $m=\Omega(\|w^\star-w_0\|^{6L+4})$) to be valid. However, verifying this condition is not feasible in general. Moreover, the bound is valid for \emph{expected} generalization gap where the expectation is taken over data sampling. Therefore computing the exact values of both sides of the above inequality is computationally exhaustive. For our experiments we consider one realization to estimate these values. 
Due to both of these reasons, the theoretical test loss and generalization loss that we present in this section should only be taken as approximations of the general behavior of the bound and not as an actual verified bound on the generalization. However, in order to reduce these impacts we conduct several experiments with varying network's width. 
\begin{figure}
\centering
\includegraphics[width=1.35in, height = 1.08in]{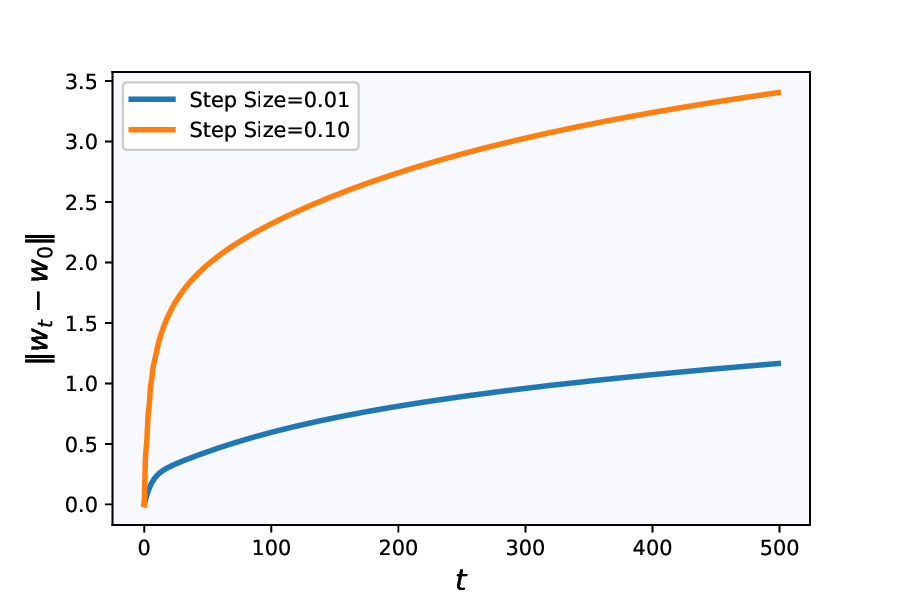}
\includegraphics[width=1.35in, height = 1.08in]{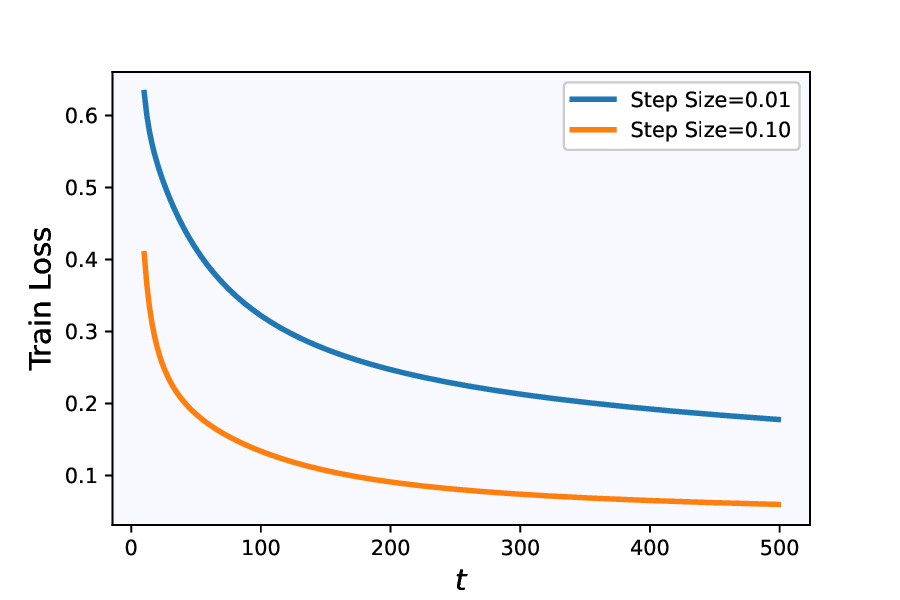}
\includegraphics[width=1.35in, height = 1.08in]{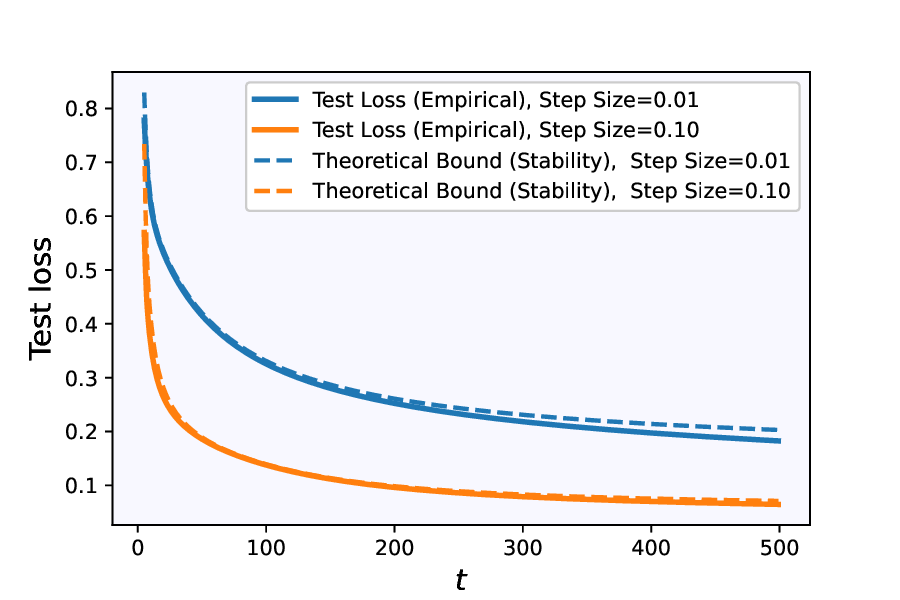}
\includegraphics[width=1.35in, height = 1.08in]{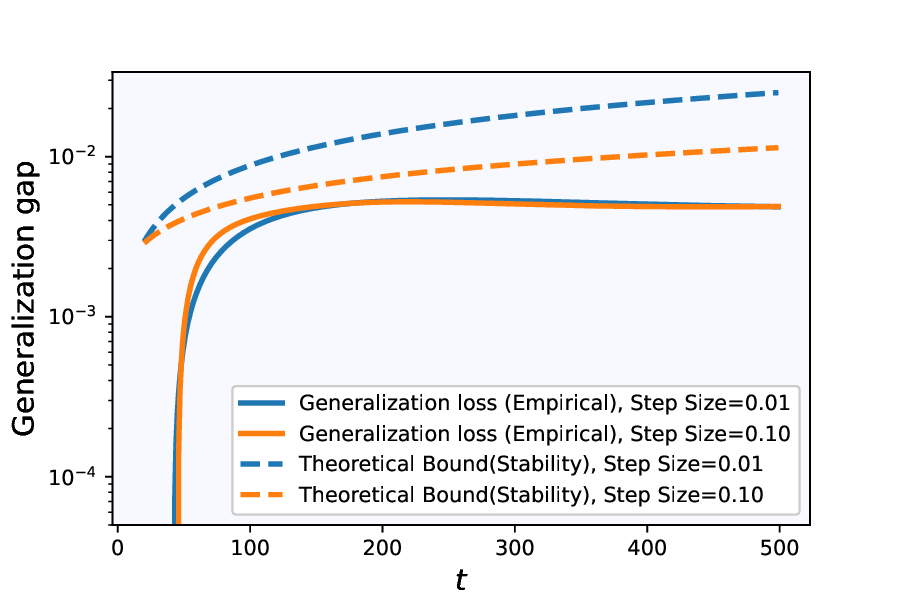}
    \caption{Iteration-based distance from initialization ($\|w_t-w_0\|$), training loss, test loss and generalization gap (i.e., test loss -- train loss) for training a two hidden-layer neural network with \emph{FashionMNIST} dataset and two choices of step-size. Here $n=12\times 10^3, m=500,$ and total number of parameters $ p \approx 6 \times 10^5$.}
    \label{fig:1}
    \centering
    \vspace{.2in}
   \includegraphics[width=1.35in, height = 1.08in]{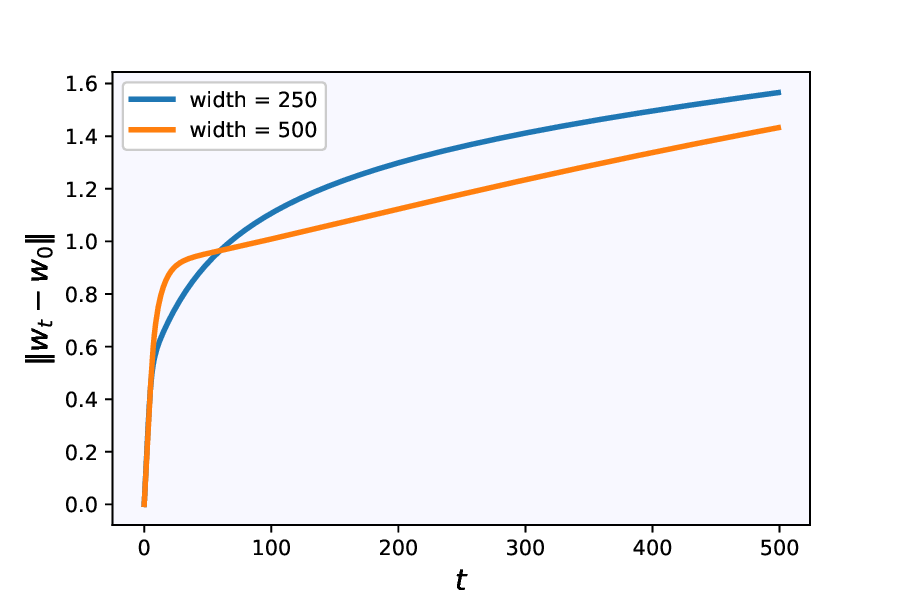}
\includegraphics[width=1.35in, height = 1.08in]{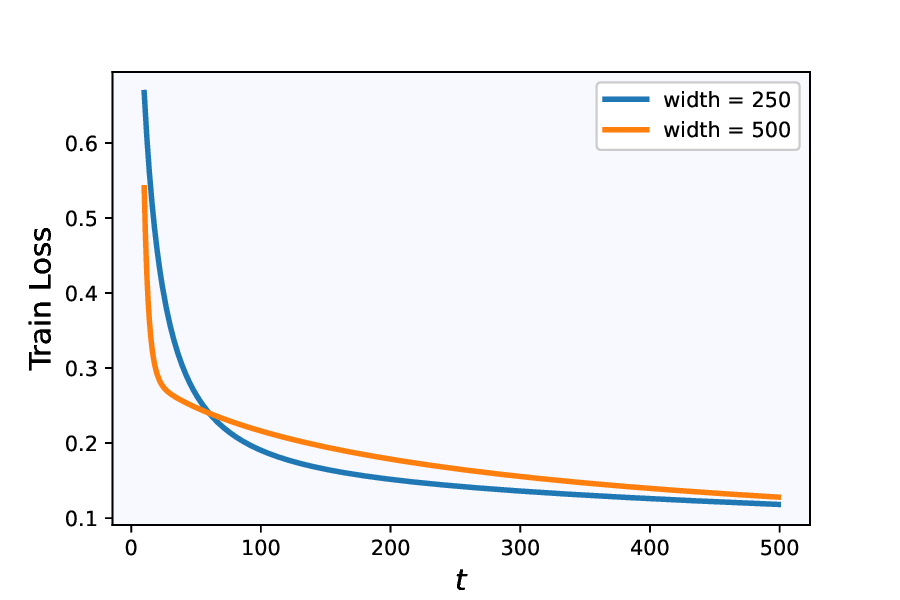}
\includegraphics[width=1.35in, height = 1.08in]{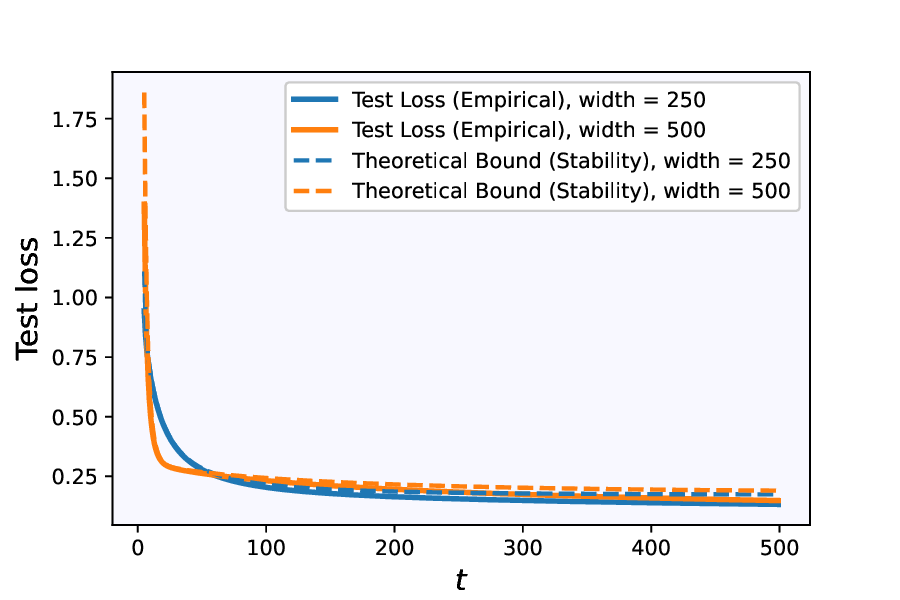}
\includegraphics[width=1.35in, height = 1.08in]{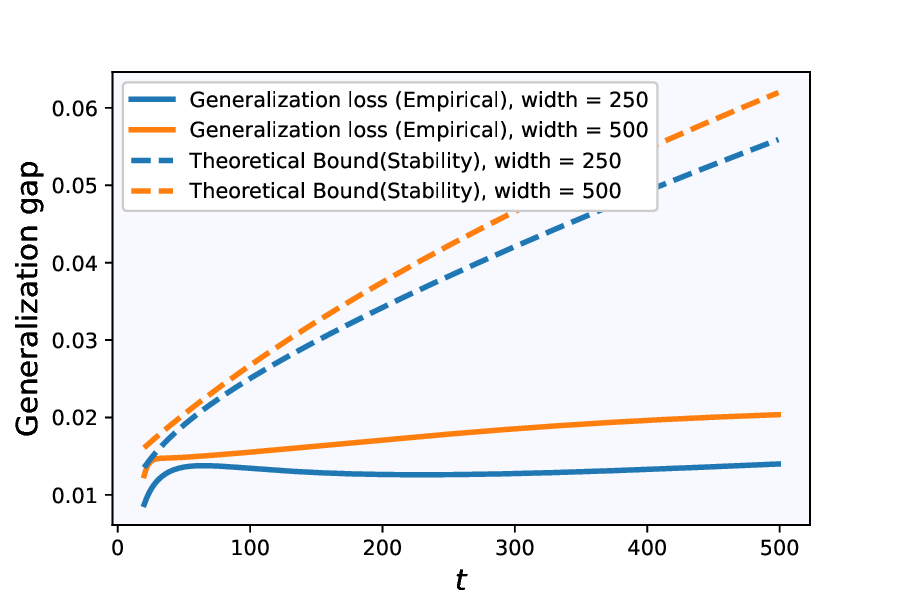}
    \caption{Iteration-based distance from initialization, training loss, test loss and generalization gap  for training a two hidden-layer neural network with \emph{FashionMNIST} dataset and $m=250, 500$. Here $n=4\times 10^3, p \approx 2\times 10^5$(blue line), $6\times 10^5$ (red line),  and $\eta=0.02 $.}
    \label{fig:2}
        \vspace{.2in}
    \centering
   \includegraphics[width=1.35in, height = 1.08in]{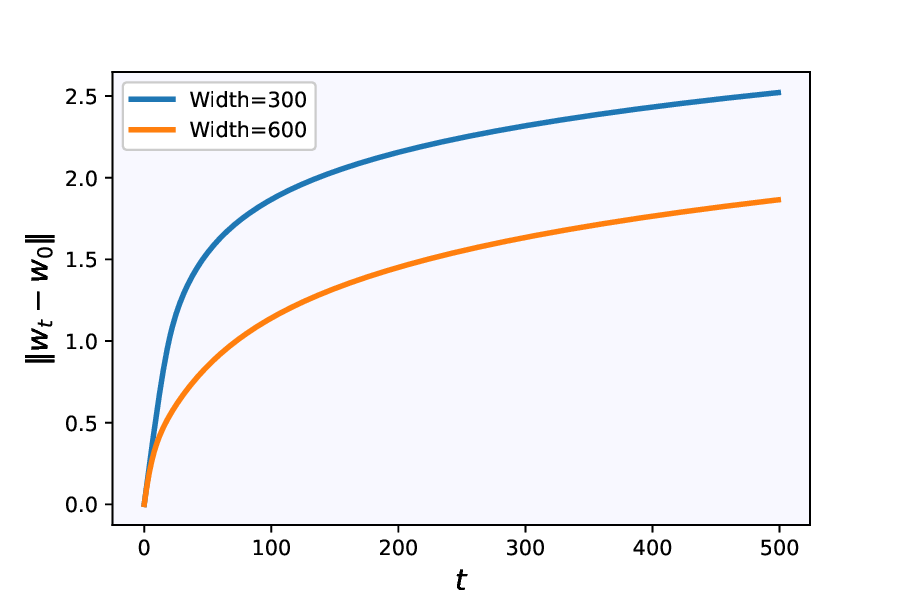}
\includegraphics[width=1.35in, height = 1.08in]{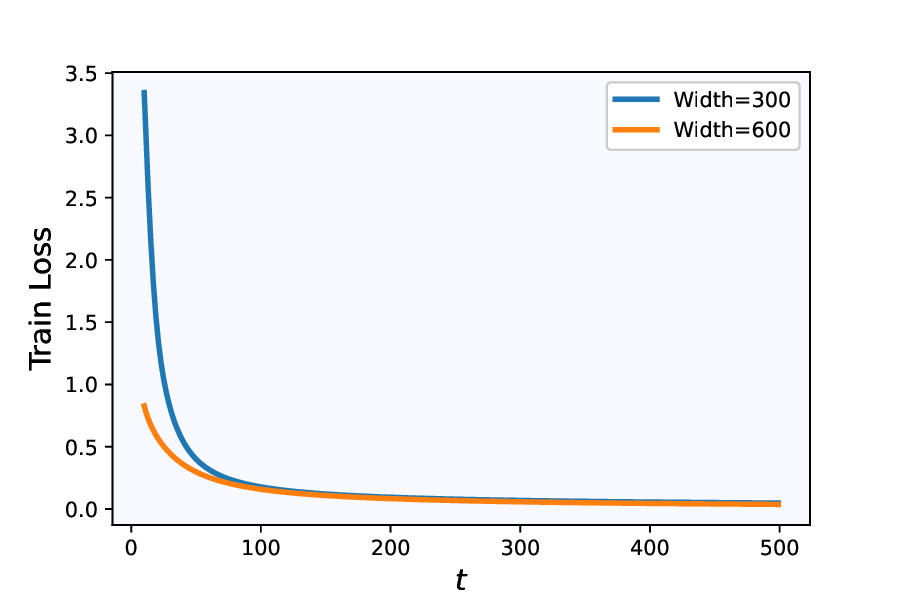}
\includegraphics[width=1.35in, height = 1.08in]{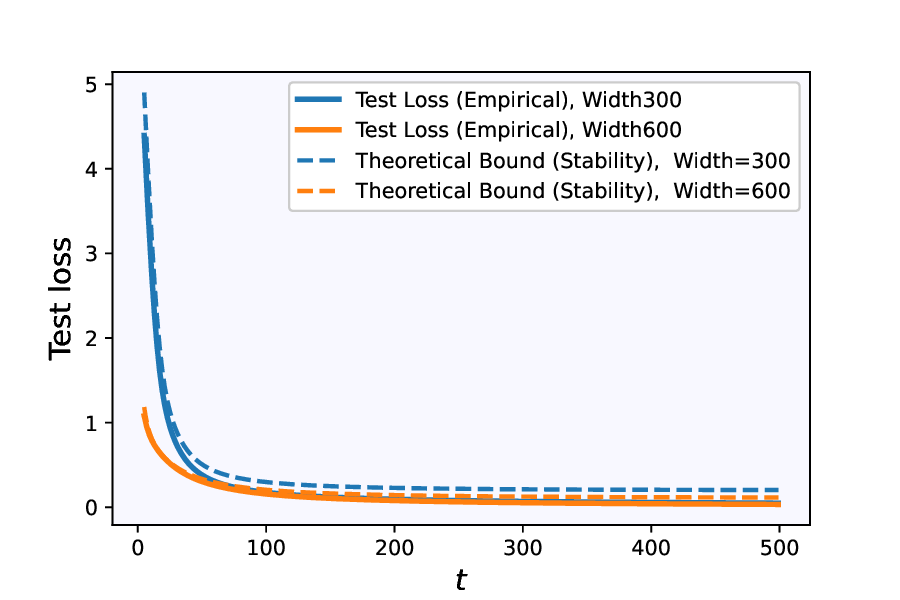}
\includegraphics[width=1.35in, height = 1.08in]{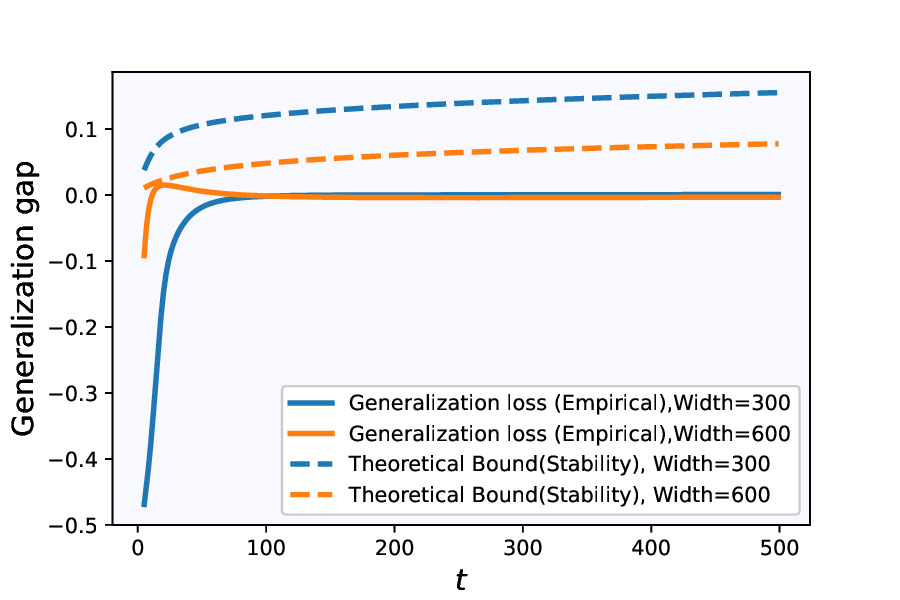}
    \caption{Iteration-based distance from initialization, training loss, test loss and generalization gap for training a two hidden-layer neural network with \emph{MNIST} dataset and $m=300,600$. Here $n=2\times 10^3, p \approx 3\times 10^5$(blue line), $8\times 10^5$ (red line) and $\eta=0.02$.}
    \label{fig:3}
\end{figure}

\begin{figure}
\centering
\includegraphics[width=1.8in,height = 1.5in]{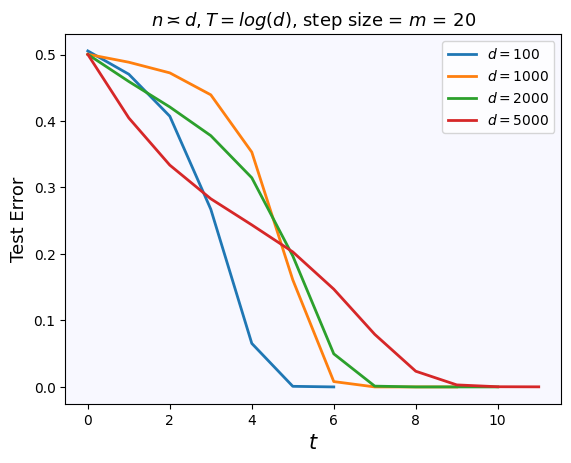}\hspace{.5in}
\includegraphics[width=1.8in,height = 1.5in]{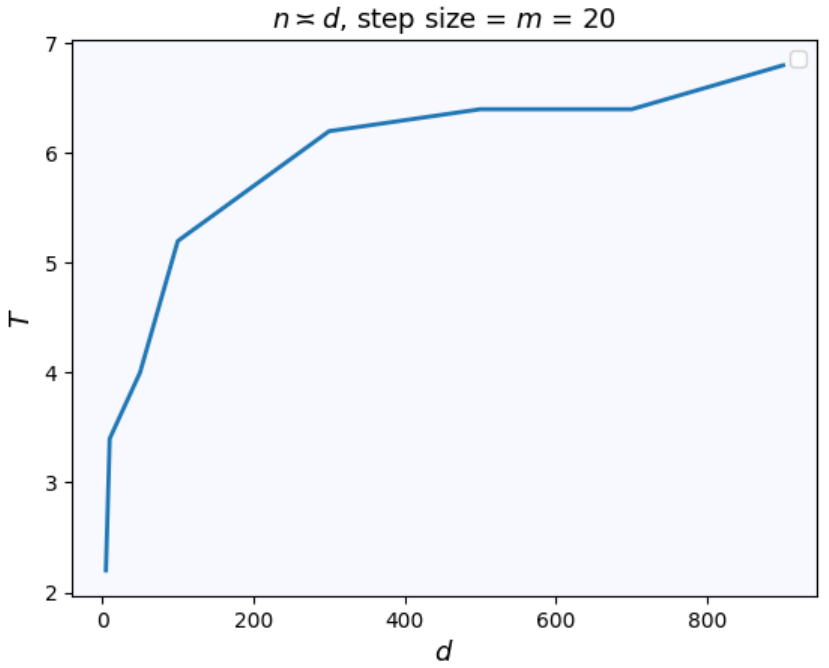}
\caption{Left: Misclassification error based on iteration in learning the $d-$dimensional XOR distribution with SGD. Right: Total number of SGD steps based on data dimension to reach approximately zero test error.}
\label{fig:4}
\end{figure}
We consider binary classification with a 2-hidden layer network with softplus activation ($\sigma(t) = \log(1+e^t)$) trained by the logistic loss function. Figure \ref{fig:1} presents train, test and generalization behavior of GD for learning a such a model with FashionMNIST dataset. The two lines in each figure correspond to $\eta=0.01, 0.1$. In the two rightmost plots, the resulting theoretical generalization and test loss curves derived from Eq. \ref{eq:stabsim2} are compared with the empirical values.  Note the good alignment between theoretical and empirical behavior for the generalization and test loss. 
%We also note that the total number of parameters for this task is remarkably large $p\approx 6\times 10^5$. However, the total distance traveled by the network weights is relatively small and of constant order. Importantly, the quantity $\|w_t-w_0\|$ is an upper-bound for the key factor $\|w^\star-w_0\|$ which characterizes the training and test loss behavior in Thms \ref{thm:train-test}-\ref{thm:train-test}.
\par
A similar behavior is observed in Figures \ref{fig:2}-\ref{fig:3}. In Figure \ref{fig:2}, we consider a similar setup but we reduce the sample size to $4000$ training data in order to allow larger test-loss behavior. The resulting plots show the training and test loss performance for two choices of $m=250$ and $m=500.$ In Figure \ref{fig:3}, we consider the MNIST dataset for $m=300$ and $m=600$ and let the sample size be $n=2000.$ It is noteworthy that the findings in both figures bear similarities, with the theoretical bounds providing non-vacuous and accurate approximations for the test loss and generalization gap.
\paragraph{Experiments on learning the XOR distribution with large step-size.}
Figure \ref{fig:4} demonstrates the test error curves associated with learning the XOR distribution according to the setting of Theorem \ref{thm:xor}. In particular, we fix $n = 6d$, $\eta = m = 20$ and set the total number of SGD steps as $T=\ceil{\log(d)}$. Note that the number of iterations required to reach perfect accuracy grows with $d$. The right side of Figure \ref{fig:4} provides further insight into the relationship between dimensionality and convergence rate. It displays the total number of SGD steps required to reach a test error below 0.01 for different values of $d$ using $n=3d$, $m=\eta=20$. The results are averages over five independent experiments and highlight the logarithmic dependence of the total SGD iterations on data dimension for achieving near-zero test error.

\section{Conclusions and Future Directions}
We explored the convergence and generalization of smooth neural networks trained with gradient methods. Our first goal in this paper was to derive generalization bounds through a new stability-based approach which had not been discussed in the vast literature of deep learning. These findings represent an improvement over previous results that either required substantial over-parameterization or provided suboptimal generalization rates that depended on the network width. For general noisy data distributions, we also derived generalization guarantees showing that GD can reach the optimal test loss. We also showed that orders of magnitude improvements in sample and computational complexity are possible by surpassing the NTK limitations and using SGD with large step-size. Several directions remain open to future research:

\begin{itemize}
\item It remains open to explore whether the minimum width conditions of Theorems \ref{thm:train-test}-\ref{thm:generalization} can be improved in terms of $\gamma$ or $n$. 

\item It is also interesting to extend the XOR analysis to the noisy setting where a fraction of the data points have corrupted labels. 

\item The feature learning phenomenon in multi-index classification tasks remains largely unexplored. While we believe the XOR setup is a must-take first step, extensions to other multi-index models can shed further light on the  strengths and limitations of neural networks. 

\item We also aim to understand the potential benefits of network depth in either the NTK regime or feature learning i.e., whether adding a single layer  can improve sample complexity or reduce the total number of SGD iterations.
%\item Can Theorem \ref{thm:xor} be extended to classification tasks where the label is determined as a multi-index function? 
\end{itemize}
%\section*{Broader Impact}
%The results presented in this work are mainly of theoretical nature. Therefore, we do not see any societal impact associated with this work. 
%\clearpage
\bibliography{main}

\begin{thebibliography}{}

\bibitem[Abbe et~al., 2022]{abbe2022merged}
Abbe, E., Adsera, E.~B., and Misiakiewicz, T. (2022).
\newblock The merged-staircase property: a necessary and nearly sufficient
  condition for sgd learning of sparse functions on two-layer neural networks.
\newblock In {\em Conference on Learning Theory}, pages 4782--4887. PMLR.

\bibitem[Ba et~al., 2022]{ba2022high}
Ba, J., Erdogdu, M.~A., Suzuki, T., Wang, Z., Wu, D., and Yang, G. (2022).
\newblock High-dimensional asymptotics of feature learning: How one gradient
  step improves the representation.
\newblock {\em Advances in Neural Information Processing Systems},
  35:37932--37946.

\bibitem[Bai and Lee, 2019]{bai2019beyond}
Bai, Y. and Lee, J.~D. (2019).
\newblock Beyond linearization: On quadratic and higher-order approximation of
  wide neural networks.
\newblock {\em arXiv preprint arXiv:1910.01619}.

\bibitem[Banerjee et~al., 2022]{banerjee2022restricted}
Banerjee, A., Cisneros-Velarde, P., Zhu, L., and Belkin, M. (2022).
\newblock Restricted strong convexity of deep learning models with smooth
  activations.
\newblock {\em arXiv preprint arXiv:2209.15106}.

\bibitem[Barak et~al., 2022]{barak2022hidden}
Barak, B., Edelman, B., Goel, S., Kakade, S., Malach, E., and Zhang, C. (2022).
\newblock Hidden progress in deep learning: Sgd learns parities near the
  computational limit.
\newblock {\em Advances in Neural Information Processing Systems},
  35:21750--21764.

\bibitem[Bartlett et~al., 2017]{bartlett2017spectrally}
Bartlett, P.~L., Foster, D.~J., and Telgarsky, M.~J. (2017).
\newblock Spectrally-normalized margin bounds for neural networks.
\newblock {\em Advances in neural information processing systems}, 30.

\bibitem[Bousquet and Elisseeff, 2002]{bousquet2002stability}
Bousquet, O. and Elisseeff, A. (2002).
\newblock Stability and generalization.
\newblock {\em Journal of Machine Learning Research (JMLR)}, 2:499--526.

\bibitem[Cao and Gu, 2019]{cao2019generalization}
Cao, Y. and Gu, Q. (2019).
\newblock Generalization bounds of stochastic gradient descent for wide and
  deep neural networks.
\newblock {\em Advances in neural information processing systems}, 32.

\bibitem[Chen et~al., 2021]{chen2020much}
Chen, Z., Cao, Y., Zou, D., and Gu, Q. (2021).
\newblock How much over-parameterization is sufficient to learn deep
  re{\{}lu{\}} networks?
\newblock In {\em International Conference on Learning Representations}.

\bibitem[Cui et~al., 2024]{cui2024asymptotics}
Cui, H., Pesce, L., Dandi, Y., Krzakala, F., Lu, Y.~M., Zdeborov{\'a}, L., and
  Loureiro, B. (2024).
\newblock Asymptotics of feature learning in two-layer networks after one
  gradient-step.
\newblock {\em arXiv preprint arXiv:2402.04980}.

\bibitem[Damian et~al., 2022]{damian2022neural}
Damian, A., Lee, J., and Soltanolkotabi, M. (2022).
\newblock Neural networks can learn representations with gradient descent.
\newblock In {\em Conference on Learning Theory}, pages 5413--5452. PMLR.

\bibitem[Du et~al., 2019]{du2019gradient}
Du, S., Lee, J., Li, H., Wang, L., and Zhai, X. (2019).
\newblock Gradient descent finds global minima of deep neural networks.
\newblock In {\em International conference on machine learning}, pages
  1675--1685. PMLR.

\bibitem[Glasgow, 2024]{glasgowsgd}
Glasgow, M. (2024).
\newblock Sgd finds then tunes features in two-layer neural networks with
  near-optimal sample complexity: A case study in the xor problem.
\newblock In {\em The Twelfth International Conference on Learning
  Representations}.

\bibitem[Golowich et~al., 2018]{golowich2018size}
Golowich, N., Rakhlin, A., and Shamir, O. (2018).
\newblock Size-independent sample complexity of neural networks.
\newblock In {\em Conference On Learning Theory}, pages 297--299. PMLR.

\bibitem[Hardt et~al., 2016]{hardt2016train}
Hardt, M., Recht, B., and Singer, Y. (2016).
\newblock Train faster, generalize better: Stability of stochastic gradient
  descent.
\newblock In {\em International Conference on Machine Learning (ICML)},
  volume~48, pages 1225--1234.

\bibitem[Jacot et~al., 2018]{jacot2018neural}
Jacot, A., Gabriel, F., and Hongler, C. (2018).
\newblock Neural tangent kernel: Convergence and generalization in neural
  networks.
\newblock In {\em Neural Information Processing Systems (NeurIPS)}, volume~31.

\bibitem[Ji and Telgarsky, 2020]{ji2019polylogarithmic}
Ji, Z. and Telgarsky, M. (2020).
\newblock Polylogarithmic width suffices for gradient descent to achieve
  arbitrarily small test error with shallow relu networks.
\newblock {\em International Conference on Learning Representations}.

\bibitem[Laurent and Massart, 2000]{laurent2000adaptive}
Laurent, B. and Massart, P. (2000).
\newblock Adaptive estimation of a quadratic functional by model selection.
\newblock {\em Annals of statistics}, pages 1302--1338.

\bibitem[Lei and Ying, 2020]{lei2020fine}
Lei, Y. and Ying, Y. (2020).
\newblock Fine-grained analysis of stability and generalization for stochastic
  gradient descent.
\newblock In {\em International Conference on Machine Learning (ICML)}, volume
  119, pages 5809--5819.

\bibitem[Li et~al., 2020]{li2020gradient}
Li, M., Soltanolkotabi, M., and Oymak, S. (2020).
\newblock Gradient descent with early stopping is provably robust to label
  noise for overparameterized neural networks.
\newblock In {\em International conference on artificial intelligence and
  statistics}, pages 4313--4324. PMLR.

\bibitem[Liu et~al., 2020]{liu2020linearity}
Liu, C., Zhu, L., and Belkin, M. (2020).
\newblock On the linearity of large non-linear models: when and why the tangent
  kernel is constant.
\newblock {\em Advances in Neural Information Processing Systems},
  33:15954--15964.

\bibitem[Liu et~al., 2022]{liu2022loss}
Liu, C., Zhu, L., and Belkin, M. (2022).
\newblock Loss landscapes and optimization in over-parameterized non-linear
  systems and neural networks.
\newblock {\em Applied and Computational Harmonic Analysis}, 59:85--116.

\bibitem[Neyshabur et~al., 2018]{neyshabur2018pac}
Neyshabur, B., Bhojanapalli, S., and Srebro, N. (2018).
\newblock A pac-bayesian approach to spectrally-normalized margin bounds for
  neural networks.
\newblock In {\em International Conference on Learning Representations}.

\bibitem[Nitanda et~al., 2019]{nitanda2019gradient}
Nitanda, A., Chinot, G., and Suzuki, T. (2019).
\newblock Gradient descent can learn less over-parameterized two-layer neural
  networks on classification problems.
\newblock {\em arXiv preprint arXiv:1905.09870}.

\bibitem[Oymak and Soltanolkotabi, 2019]{oymak2019overparameterized}
Oymak, S. and Soltanolkotabi, M. (2019).
\newblock Overparameterized nonlinear learning: Gradient descent takes the
  shortest path?
\newblock In {\em International Conference on Machine Learning}, pages
  4951--4960. PMLR.

\bibitem[Richards and Rabbat, 2021]{richards2021learning}
Richards, D. and Rabbat, M. (2021).
\newblock Learning with gradient descent and weakly convex losses.
\newblock In {\em International Conference on Artificial Intelligence and
  Statistics}, pages 1990--1998. PMLR.

\bibitem[Schliserman and Koren, 2022]{schliserman2022stability}
Schliserman, M. and Koren, T. (2022).
\newblock Stability vs implicit bias of gradient methods on separable data and
  beyond.
\newblock In {\em Conference on Learning Theory (COLT)}, volume 178, pages
  3380--3394.

\bibitem[Shalev-Shwartz and Ben-David, 2014]{shalev2014understanding}
Shalev-Shwartz, S. and Ben-David, S. (2014).
\newblock {\em Understanding machine learning: From theory to algorithms}.
\newblock Cambridge university press.

\bibitem[Taheri and Thrampoulidis, 2024]{taheri2024generalization}
Taheri, H. and Thrampoulidis, C. (2024).
\newblock Generalization and stability of interpolating neural networks with
  minimal width.
\newblock {\em Journal of Machine Learning Research}, 25(156):1--41.

\bibitem[Telgarsky, 2023]{telgarsky2023feature}
Telgarsky, M. (2023).
\newblock Feature selection and low test error in shallow low-rotation relu
  networks.
\newblock In {\em The Eleventh International Conference on Learning
  Representations}.

\bibitem[Vershynin, 2018]{vershynin2018high}
Vershynin, R. (2018).
\newblock {\em High-dimensional probability: An introduction with applications
  in data science}, volume~47.
\newblock Cambridge university press.

\bibitem[Wang et~al., 2023]{wang2023generalization}
Wang, P., Lei, Y., Wang, D., Ying, Y., and Zhou, D.-X. (2023).
\newblock Generalization guarantees of gradient descent for multi-layer neural
  networks.
\newblock {\em arXiv preprint arXiv:2305.16891}.

\bibitem[Wei et~al., 2019]{wei2019regularization}
Wei, C., Lee, J.~D., Liu, Q., and Ma, T. (2019).
\newblock Regularization matters: Generalization and optimization of neural
  nets vs their induced kernel.
\newblock {\em Advances in Neural Information Processing Systems}, 32.

\end{thebibliography}
\bibliographystyle{apalike}
\clearpage
%\onecolumn

\appendix
\section*{Appendix}
\section{Preliminaries : Hessian and Gradient norm}
\subsection{ Hessian norm}\label{sec:hess}

% Assume a deep network with initialization $w_{0,ij}\sim N(0,\sigma_0^2)$ where $\sigma_0 = \frac{\sigma_1}{2+4\sqrt{\frac{\log (m)}{m}})}$. Define 
% \bea
% \mathcal{B}_\rho(w_0):= \{w\in R^p: \|W^{(\ell)} - W_0^{(\ell)}\|_2\le \rho, \forall \ell \in [L]\}
% \eea
% For our analysis we can replace the spectral norm with the frobenius norm as the Frobenius norm is always larger than the spectral norm. 
% We have by theorem 4.1. in \cite{banerjee2022restricted}, that for $w\in\mathcal{B}_\rho(w_0)$, w.p., $1-2(L+1)/m$ it holds that 
% \bea
% \max_{i\in[n]} \|\nabla^2\Phi (w,x_i)\|_2  =  O(\frac{poly(L)(1+\gamma^{2L})}{\sqrt{m}})
% \eea
% where $\gamma:=\sigma_1 + \rho/\sqrt{m}$.
% refer to \cite{liu2020linearity}                    
It is essential to our analysis to obtain Hessian's spectral norm throughout the iterates of GD. With our weights' initialization in place, one can bound the Hessian's spectral norm based on the euclidean distance of the weights from initialization and network's size as in the following assumption. 
\begin{assumption}[Hessian's Spectral Norm]\label{ass:hess}
The spectral norm of the model's Hessian for any $w\in\R^p,\rho>0$ such that $\|w-w_0\|\le \rho$ is bounded as $\|\nabla^2\Phi(w,\cdot)\|_2\le \beta_L \frac{\rho^{3L}}{\sqrt{m}}$ for some constant $\beta_L$ that depends only on $L$.
\end{assumption}
It can be shown \cite[Lemma \ref{lem:hess}]{liu2020linearity} that for standard Gaussian initialization and bounded data, Assumption \ref{ass:hess} holds with high-probability over initialization for $\beta_L= e^{O(L)}$.
%%%

For the spectral norm of the model's Hessian which is needed by our analysis, we use the following bound which bounds the spectral norm based on the distance from initialization. 
\begin{lemma}[Hessian's spectral norm \cite{liu2020linearity}]\label{lem:hess}
Assume the defined $L$ layer neural network architecture with standard Gaussian initialization i.e., $w_{_{0,{ij}}}^{(\ell)} \simiid N(0,1)$. Let $\|w-w_0\| \le \rho$. Then with high probability
\bea\nn
\|\nabla^2 \Phi(w,x)\|_2 \le \beta_L\frac{\rho^{3L}}{\sqrt{m}}
\eea
for some constant $\beta_L$ that only depends on $L.$
\end{lemma}                                           
As mentioned in the main body, the result above guarantees that Assumption \ref{ass:hess} holds with high probability. 
\subsection{Gradient norm at initialization}\label{sec:gradnorm}
Recall the model for Theorems \ref{thm:train-test}-\ref{thm:generalization},
\begin{align*}
\Phi(w,x)&:= \frac{1}{\sqrt m} \langle W^{(L+1)}, x^{(L)}\rangle\\
x_i^{(\ell+1)}&:=\frac{1}{\sqrt m} \sigma(\langle W_i^{(\ell+1)}, x^{(\ell)}\rangle )\;\;\; \forall \ell\in\{1,\cdots,L-1\},\forall i\in\{1,\cdots,m\}\\
x_i^{(1)} &:= \sigma(\langle x,W^{(1)}_i\rangle)\;\;\;\;\; \forall i\in\{1,\cdots,m\}
\end{align*}
%assume $w_{ij}\sim N(0,\sigma^2/\alpha)$ and normalized last layer and \red{it is trained}. Then using Lemma 4.1 from \cite{banerjee2022restricted} we find that $\|\nabla \Phi(w,x)\| \le \beta$ where
\bea
%\beta^2=\gamma^{2L}+\frac{1}{m}\left(1+\rho_1\right)^2 (L+1) \gamma^{2L-2}, \gamma=\sigma+\frac{\rho}{\sqrt{m}}
\eea
%if the last layer is not trained the first term above disappears and the Lipschitz parameter is decreasing with m. \red{this does not make sense at all, perhaps their results are vacuous in any interesting regime.}

Let us consider the standard Gaussian initialization. We have

% then $|w_{ij}|<t$ w.p. $1- Lm^2 e^{-t^2/\zeta^2}=1-\exp(-t^2/\zeta^2+\log(Lm^2))$ uniformly for all weights.  Also, $x_i^{(1)} \le |\langle x,w^{(1)}_i\rangle| < t$ w.p. $1-m e^{-t^2/\zeta^2}$. Therefore,
% \bea
% x_i^{(2)} = \frac{1}{\sqrt{m}} \sigma(\langle w_i^{(2)},x^{(1)}\rangle) \le \frac{1}{\sqrt{m}} \langle w_i^{(2)},x^{(1)}\rangle 
% \eea
\bea
\frac{\partial \Phi(w,x)}{\partial W^{(\ell+1)}_{ij}} = \frac{1}{\sqrt m} x^{(\ell)}_j \sigma'(\langle W_i^{(\ell+1)},x^{(\ell)}\rangle)\frac{\partial \Phi(w,x)}{\partial x^{(\ell+1)}_{i}}
\eea
% \bea
% \|x_j^{(\ell)}\|^2 = \frac{1}{m}\sigma^2 (\langle w_j^{(\ell)},x^{(\ell-1)}\rangle)\le \frac{1}{m}\ell_\sigma^2\|w_j^{(\ell)}\|^2\|x^{(\ell-1)}\|^2
% \eea
% \bea
% \|x^{(\ell)}\|^2 &\le \frac{1}{m}\ell_\sigma^2\|w^{(\ell)}\|^2\|x^{(\ell-1)}\|^2\\
% &\le R \prod_{k=1}^\ell \frac{1}{m} \|w^{(k)}\|^2 
% \eea
% also, one can deduce that
% \bea
% |\frac{\partial \Phi(w,x)}{\partial x_i^{(\ell+1)}}| &\le \frac{1}{\sqrt  m}\sum_{j=1}^{ m} |\frac{\partial \Phi(w,x)}{\partial x_j^{(\ell+2)}}| |w_{ji}^{(\ell+2)}|\\
% &\le \prod_{k=\ell+1}^{L} \frac{\|w^{(k+1)}\|}{\sqrt m}
% \eea
% thus,
% \bea
% |\frac{\partial \Phi(w,x)}{\partial w^{(\ell+1)}_{ij}}| &\le \frac{1}{\sqrt m}|x_j^{(\ell)}|\cdot|\frac{\partial \Phi(w,x)}{\partial x^{(\ell+1)}_{i}}|\\
% &\le \sqrt{R} \|w_{\cdot i}^{(\ell+2)}\| \|w_{j}^{(\ell)}\| \prod_{k=1}^L \frac{1}{\sqrt  m}\prod_{k=1,(k\neq\ell,\ell+1,\ell+2)}^{L+1} \|w^{(k)}\|
% \eea
% summing over $i,j$ and letting $\tau:=\min_{k\in\{1,\cdots,L+1\}}\|w^{(k)}\|$,

% \bea
% ||\frac{\partial \Phi(w,x)}{\partial w^{(\ell+1)}}||&\le \frac{\sqrt{R}}{\|w^{(\ell+1)}\|}\prod_{k=0}^L \frac{\|w^{(k+1)}\|}{\sqrt m}\\
% &\le\frac{\sqrt{R}}{\tau} \prod_{k=0}^L \frac{\|w^{(k+1)}\|}{\sqrt m}
% \eea
% therefore
% \bea
% \|\nabla_1\Phi(w,x)\| \le \frac{\sqrt{R}\sqrt{L+1}}{\tau} \prod_{k=0}^L \frac{\|w^{(k+1)}\|}{\sqrt m}
% \eea
Therefore, 
\bea
\|\frac{\partial \Phi(w,x)}{\partial W^{(\ell+1)}}\|^2 =\sum_{ij}\|\frac{\partial \Phi(w,x)}{\partial W^{(\ell+1)}_{ij}}\|^2 \le \frac{1}{m}\|x^{(\ell)}\|^2\|\frac{\partial \Phi(w,x)}{\partial x^{(\ell+1)}}\|^2
\eea
by \cite[Lemmas F.3 and F.5]{liu2020linearity} :
\bea
\|x^{(\ell)}\| \le c_0^\ell \sqrt{m},~~~~ \|\frac{\partial \Phi}{\partial x^{(\ell)}} \|\le c_0^{L-\ell+1}
\eea
Therefore $\|\frac{\partial \Phi(w_0,x)}{\partial W^{(\ell+1)}}\|^2 \le c_0^{2L}$ and 
$$\|\nabla \Phi(w_0,x)\| = (\sum_{\ell=0}^L \|\frac{\partial \Phi(w_0,x)}{\partial W^{(\ell+1)}}\|^2)^{1/2}\le \sqrt{L+1} c_0^{L}.$$

where $c_0$ is any constant that bounds the spectral norm of initialization weights with high probability, i.e., $\frac{\|W_0\|_2}{\sqrt{m}} \le c_0.$ Specially, well-known bounds on the spectral norm of Gaussian matrices guarantee that $c_0$ is constant with high proabability \cite{vershynin2018high}.

Moreover, based on \cite{liu2020linearity}, if $\|w-w_0\|=\rho$ the spectral norm is bounded by,
\bea
\|\nabla^2\Phi(w,x)\|_2\le \beta_L\frac{\rho^{3L}}{\sqrt{m}}
\eea
where $\beta_L$ is a depth-only dependent variable derived to be $\beta_L  = e^{O(L)}$.

Now, we use the above relation on Hessian to obtain a bound on the gradient:

\bea
\frac{\|\nabla \Phi(w^*,x)-\nabla\Phi(w_0,x)\|}{\|w^*-w_0\|} \le \max_{w\in[w_0,w^*]} \|\nabla^2 \Phi(w,x)\|_2
\eea
therefore if $w^\star\in\mathcal{B}_\rho(w_0)$,
\bea
\|\nabla \Phi(w^*,x)-\nabla\Phi(w_0,x)\| \le\frac{\beta_L \rho^{3L+1}}{\sqrt{m}}.
\eea

Thus,
\bea
\|\nabla\Phi(w^*,x)\|- \|\nabla\Phi(w_0,x)\| \le \|\nabla \Phi(w^*,x)-\nabla\Phi(w_0,x)\| \le \frac{\beta_L \rho^{3L+1}}{\sqrt{m}}.
\eea
Thus w.h.p. 
\bea
\max_{i\in[n]}\|\nabla \Phi(w^*,x_i)\| \le \max_{i\in[n]}\|\nabla\Phi(w_0,x_i)\| + \frac{\beta_L\rho^{3L+1}}{\sqrt{m}}
\eea

Let us denote the gradient and hessian bounds on $\Phi$ with $C_1$ and $C_2$, respectively. Note that based on our derivations above it holds that $C_1 \le  \frac{\beta_L \rho^{3L+1}}{\sqrt{m}} + G_0$ 
and $C_2 \le \frac{\beta_L \rho^{3L}}{\sqrt{m}},$ where $G_0$ is the Lipschtiz parameter at initialization. With these parameters we can now obtain the Lipschitz and Smoothness parameter as well as the smallest eigenvalue of the objective along the gradient descent path. This is derived in the next lemma.

\subsection{Objective's Gradient and Hessian along the GD path}

\begin{lemma}\label{lem:gradients_and_hessians_bounds_general}

The following are true for the loss gradient and Hessian:
\begin{enumerate}
  
\item $\|\nabla \Fh(w)\|\le C_1 \, \Fh(w).$

\item $\|\nabla^2 \Fh(w)\|_2 \leq  (C_2+ (C_1)^2) \Fh(w).$

\item $\lambda_\min\left(\nabla^2 \widehat F(w)\right) \geq - C_2 \Fh(w)$.

\end{enumerate}
where we define $C_1 := G_0 + \beta_L \|w-w_0\|^{3L+1}/\sqrt{m},$ $C_2 := \beta_L \|w-w_0\|^{3L}/\sqrt{m}$ and $G_0$ is the Lipschitz parameter of the network at initialization i.e., $\|\nabla \Phi(w_0,\cdot)\|\le G_0.$
\end{lemma}

\begin{proof} 
The loss gradient is derived as follows,
    \begin{align*}
    \nabla \widehat F(w) &= \frac{1}{n} \sum_{i=1}^n f^\prime(y_i \Phixi)y_i \nabla \Phi(w,x_i)
    \end{align*}
    Recalling that $y_i\in\{\pm 1\}$, we can write
\bea
\Big\|\nabla \widehat F(w)\Big\| &= \frac{1}{n} \Big\| \sum_{i=1}^n f'(y_i \Phi(w,x_i))y_i \nabla \Phi (w,x_i)\Big\|\nn\\
&\le \frac{1}{n} \sum_{i=1}^n |f'(y_i \Phi(w,x_i))| \Big\|\nabla \Phi (w,x_i)\Big\|.\nn\\
&\le C_1 \, \hf(w). 
\eea
For the Hessian of loss, note that
    \begin{align}\label{eq:lossh}
    \nabla^2 \widehat F(w) &= \frac{1}{n} \sum_{i=1}^n f^{\prime\prime}(y_i \Phixi) \nabla \Phi(w,x_i)\nabla \Phi(w,x_i)^\top + 
     f^{\prime}(y_i \Phixi)y_i \nabla^2 \Phi(w,x_i).
    \end{align}
It follows that
\bea
\Big\|\nabla^2 \widehat F(w)\Big\|_2 &= \left\|\frac{1}{n} \sum_{i=1}^n f^{\prime}(y_i \Phi(w,x_i))y_i \nabla^2 \Phi (w,x_i) + f^{\prime\prime}(y_i \Phi(w,x_i)) \nabla\Phi(w,x_i) \nabla_1 \Phi(w,x_i)^\top \right\|\nn\\
&\le \frac{1}{n}\sum_{i=1}^n |f'(y_i \Phi(w,x_i))| \Big\|\nabla^2 \Phi (w,x_i) \Big\| + |f''(y_i \Phi(w,x_i))|\Big\|\nabla \Phi (w,x_i) \nabla\Phi(w,x_i)^\top \Big\|\nn\\
&\le  \frac{1}{n}\sum_{i=1}^n |f'(y_i \Phi(w,x_i))| \Big\|\nabla^2 \Phi (w,x_i) \Big\| + |f''(y_i \Phi(w,x_i))|\Big\|\nabla \Phi (w,x_i)\Big\|^2\nn\\
&\le C_2 \Fh(w)+ (C_1)^2 \Fh(w).
\label{eq:heslast} 
\eea
To lower-bound the minimum eigenvalue of Hessian, note that the loss function $f$ is convex and thus $f''(\cdot)\ge 0$. Therefore the first term in Eq. \ref{eq:lossh} is positive semi-definite and the second term can be lower-bounded as follows,
\bea
\lambda_{\min} (\nabla ^2 \widehat F(w)) &\ge -\left \|\frac{1}{n} \sum_{i=1}^n y_i f'(y_i\Phi(w,x_i)) \nabla^2 \Phi(w,x_i)\right\|\nn\\
&\ge -\frac{1}{n} \sum_{i=1}^n |y_i f'(y_i\Phi(w,x_i))| \Big\|\nabla^2 \Phi(w,x_i)\Big\|\nn\\
&\ge -C_2 \Fh(w).\nn
\eea
    
\end{proof}
\section{Proof of Theorem \ref{thm:train-test}}\label{sec:proof_thm1}
With the bounds on Hessian and Gradient from the last section, we are ready to prove Theorem \ref{thm:train-test} for obtaining the training and generalization rates of deep nets. We start with proving the training guarantees in the following theorem. 
\subsection{Proof for training loss}
\begin{theorem}[Convergence of GD in deep nets]\label{thm:b1}
Let the descent lemma hold. Fix a $T$ and assume the target weights vector $w^\star\in\R^p$ that obtain small training loss such that 
\bea\nn
\rho^\star\geq \max\left\{\sqrt{\eta T \widehat F(w^\star)},\sqrt{\eta \widehat F(w_0)}\right\}.
\eea
where $\rho^\star:=\|w^\star-w_0\|$. Moreover, assume the width $m$ is large enough such that it satisfies,
\bea\nn
m \ge 4\beta_L^2 (3 {\rho^\star})^{6L+4}.
\eea
The the training loss satisfies with high probability over initialization,
\bea\nn
\frac{1}{T} \sum_{t=1}^T \widehat{F}\left(w_t\right) \le 2 \Fh(w^\star)+\frac{2{\rho^\star}^2}{\eta T}.
\eea
\end{theorem}
To begin the proof, we first state our descent lemma.
\begin{lemma}[Descent lemma]\label{lem:des}
Consider the notation as in Lemma \ref{lem:gradients_and_hessians_bounds_general}. Assume for all $t\in[T]$ we have $\eta\le 1/((C_1)^2+C_2)$, then 
$$\widehat{F}\left(w_{t+1}\right) \le \widehat{F}\left(w_t\right)-\frac{\eta}{2}\left\|\nabla \widehat{F}\left(w_t\right)\right\|^2 .$$

\end{lemma}
\begin{proof}
By Taylor expansion and the bound on smoothness parameter of the objective, we immediately deduce:
\begin{align*}
\widehat{F}\left(w_{t+1}\right) & \leq \widehat{F}\left(w_t\right)-\eta\left\|\nabla \widehat{F}\left(w_t\right)\right\|^2+\eta^2\frac{ (C_1)^2+C_2}{2}\left\|\nabla \widehat{F}\left(w_t\right)\right\|^2 \\
& \leq \widehat{F}\left(w_t\right)-\frac{\eta}{2}\left\|\nabla \widehat{F}\left(w_t\right)\right\|^2 
\end{align*}
This completes the proof of the lemma. 
\end{proof}

Define $\rho(t):= \|w_t-w_0\|$ and $\rho^* = \|w^\star-w_0\|$. Also, define $\widetilde \rho$ as any constant such that 
$$\tilde\rho > \max\{\rho(T-1),\rho(T-2),\cdots,\rho(1),\rho^\star\}.$$ 
Then for any $w\in[w_t,w^\star]$ such that $w=\alpha w_t + (1-\alpha)w^\star$ it holds that 
$$
\|w-w_0\|\le \|\alpha w_t + (1-\alpha)w^\star-w_0\|\le \alpha \|w_t-w_0\| + (1-\alpha)\|w^\star-w_0\| \le \tilde\rho,
$$
and thus for any $w\in[w_t,w^\star]$
\bea
\lambda_\min\left(\nabla^2 \widehat F(w)\right) \ge - \widetilde C \Fh(w)
\eea
where $\widetilde C :=\frac{\beta_L \tilde \rho^{3L}}{\sqrt{m}}$, i.e., the smoothness parameter with respect to the distance $\tilde \rho$ and as computed by Lemma \ref{lem:gradients_and_hessians_bounds_general}. Therefore, by Taylor remainder theorem, for any fixed $w^\star$ there exists $w'_t\in[w_t,w^\star]$ such that .
\begin{align*}
    \widehat F(w^\star) &\geq \widehat F(w_t) + \left\langle\nabla \widehat F(w_t),w^\star-w_t\right\rangle + \frac{1}{2}\lamin{\nabla^2 \widehat F(w_t')} \|w^\star-w_t\|^2
    \\
    &\geq \widehat F(w_t) + \left\langle\nabla \widehat F(w_t),w^\star-w_t\right\rangle - \frac{1}{2}\frac{\beta_L \tilde\rho^{3L}}{{\sqrt{m}}} \widehat F(w_t') \,\|w^\star-w_t\|^2.
\end{align*}
As a direct result of the above inequality, we have derived that for any $w^\star$,
\begin{align*}
    \widehat F(w^\star)\geq \widehat F(w_t) + \left\langle \nabla \widehat F(w_t),w^\star-w_t\right\rangle - \frac{1}{2}\frac{\beta_L\tilde\rho^{3L}}{{\sqrt{m}}} \max_{w'_t\in[w_t,w^\star]} \widehat F(w_t') \,\|w-w_t\|^2.
\end{align*}
% And further assuming $|f^\prime(t)|\leq f(t)$:
% \begin{align*}
%     \widehat F(w)\geq \widehat F(w_t) + \inp{\nabla \widehat F(w_t)}{w-w_t} - \frac{1}{2}\frac{LR^2}{{\sqrt{m}}} \max_{\alpha\in[0,1]}\widehat F(\wa) \,\|w-w_t\|^2
% \end{align*}
Putting this in the relation from descent lemma (Lemma \ref{lem:des}) followed by completion of squares using $w_{t+1}-w_t=-\eta\nabla \widehat F(w_t)$ gives
\begin{align}
    \widehat F(w_{t+1})&\leq \widehat F(w^\star) - \left\langle\nabla \widehat F(w_t),w^\star-w_t\right\rangle - \frac{\eta}{2}\left\|\nabla \widehat F(w_t)\right\|^2+\frac{1}{2}\frac{\beta_L\tilde\rho^{3L}}{{\sqrt{m}}} \max_{w'_t\in[w_t,w^\star]}\widehat F(w_t') \,\|w^\star-w_t\|^2\nn
    \\
    &= \widehat F(w^\star) +\frac{1}{\eta}\left(\|w^\star-w_t\|^2-\|w^\star-w_{t+1}\|^2 \right) +\frac{1}{2}\frac{\beta_L \tilde\rho^{3L}}{{\sqrt{m}}} \max_{w'_t\in[w_t,w^\star]}\widehat F(w_t') \,\|w^\star-w_t\|^2.\label{eq:descent with weak convexity}
\end{align}

Telescoping the above display for $t=0,\ldots,{T-1}$, we deduce that,

\bea\label{eq:trainloss_gen}
\frac{1}{T} \sum_{t=1}^T \widehat{F}\left(w_t\right) \leq \widehat{F}(w^\star)+\frac{\left\|w^\star-w_0\right\|^2}{\eta T}+ \frac{\beta_L\tilde\rho^{3L}}{2T} \sum_{t=0}^{T-1} \max _{w'_t\in[w_t,w^\star]} \widehat{F}\left(w_t'\right)\left\|w^\star-w_t\right\|^2.
\eea
In order to proceed with the above relation, it is necessary to further simplify the terms $\max _{w'_t\in[w_t,w^\star]} \widehat{F}\left(w_t'\right)$. In doing so, we utilize the following Generalized Local Quasi-convexity (GLQC) property from \cite{taheri2024generalization}.
\begin{proposition}[Prop. 8 in \cite{taheri2024generalization}]\label{prop:GLQCapp}
 Let $w_1,w_2\in\R^{p}$ be two arbitrary points. Suppose $\hf:\R^{p}\rightarrow\R$ satisfies the \sbwc property with parameter $\kappa$ along the line $[w_1,w_2]$. Assume $\left\|w_1-w_2\right\|\leq D$ where $D<\sqrt{2/\kappa}$. Then,
\bea
\max_{v\in[w_1,w_2]} \widehat F(v)\le \tau\cdot \max\{\widehat F(w_1),\widehat F(w_2)\}.
\eea
where $\tau:=\left(1-\kappa D^2/2\right)^{-1}$.
\end{proposition}

For our objective we have the minimum eigenvalue of hessian being lower bounded by $\kappa  = \widetilde C:=\beta_L \tilde\rho^{3L}$. Therefore for $w_{t},w^\star$, if $\|w^\star-w_t\|^2 <D^2<1/2\widetilde C$,  then the GLQC property holds with $\tau = 4/3$. Recall that by the defnition of $\tilde\rho$ it holds that $\|w^\star-w_t\|\le \tilde\rho$. Thus, to guarantee the conditions of the proposition above are satisfied it suffices to hold that 
\bea
\sqrt{m}\ge 2 \beta_L\tilde\rho^{3L+2}.
\eea
which is exactly the width condition of our theorem. With this condition on width and the resulting simplifications from the proposition, we can write: 
\begin{align*}
\frac{1}{T} \sum_{t=1}^T \widehat{F}\left(w_t\right) &\leq \widehat{F}(w^\star)+\frac{\left\|w^\star-w_0\right\|^2}{\eta T}+\frac{\beta_L \tilde\rho^{3L}}{\sqrt{m}} \frac{2}{3T} \sum_{t=0}^{T-1} \max  \{\Fh\left(w^\star\right),\Fh(w_t)\}\left\|w^\star-w_t\right\|^2,\\
&\le \widehat{F}(w^\star)+\frac{\left\|w^\star-w_0\right\|^2}{\eta T}+ \frac{1}{3T} \sum_{t=0}^{T-1} \max  \{\Fh\left(w^\star\right),\Fh(w_t)\}\\
&\le \frac{4}{3}\widehat{F}(w^\star)+\frac{\left\|w^\star-w_0\right\|^2}{\eta T}+ \frac{1}{3T} \sum_{t=0}^{T-1}\Fh(w_t)
\end{align*}
Thus,
\bea
\frac{1}{T} \sum_{t=1}^T \widehat{F}\left(w_t\right) \le 2 \Fh(w^\star)+\frac{3\left\|w^\star-w_0\right\|^2}{2\eta T}+ \frac{1}{2T} \Fh(w_0)
\eea

Thus far, we have characterized the convergence under the condition that $\sqrt{m}\ge 2\beta_L\tilde\rho^{3L+2} = 2\beta_L (\max\{\|w_{T-1}-w_0\|,\cdots,\|w_1-w_0\|,\|w^\star-w_0\|\})^{3L+2}.$ It remains to characterize the above condition on width independent of the iteration number. In fact, in what follows we are able to show it is sufficient to reduce this condition such that it depends on only $\|w^\star-w_0\|$. We do it by an iterate-wise norm bound on $\|w^\star-w_t\|$ based on $\|w^\star-w_0\|$. 

Recall that we had,
\bea
\Fh(w_{t+1}) \le \widehat F(w^\star) +\frac{1}{\eta}\left(\|w^\star-w_t\|^2-\|w^\star-w_{t+1}\|^2 \right) +\frac{1}{2}\widetilde C\max_{w'_t\in[w_t,w^\star]}\widehat F(w_t') \,\|w^\star-w_t\|^2.
\eea
where recall that $\widetilde C=\frac{\beta_L \tilde \rho^{3L}}{\sqrt{m}},$ and $\widetilde \rho:= \max\{\rho(1),\rho(2),...\rho(T-1),\rho^\star\}$.

Hence,
\bea\label{eq:induct}
\|w_{t+1}-w^\star\|^2 \le \|w_t -w^\star\|^2 + \eta \Fh(w^\star) + \frac{\eta}{2} \frac{\beta_L \tilde \rho^{3L}}{\sqrt{m}}  \max_{w'_t\in[w_t,w^\star]}\widehat F(w_t') \,\|w^\star-w_t\|^2.
\eea
The following lemma proves through an induction-based argument that choosing $\tilde\rho = 3\rho^\star$ and the corresponding width condition based on this $\tilde\rho$ is sufficient to ensure that $\rho(t)\le \tilde\rho$ for all iterates $t\in[T]$; which in turn guarantees the above argument for convergence to be valid.
%or with $A_t:= \|w_t-w^\star\|^2$:
%\bea
%A_{t+1} \le A_t + \eta \Fh(w^\star) + \frac{\eta}{2} \widetilde C A_t\max_{\alpha\in[0,1]}\widehat F(w_{\alpha t})
%\eea

\begin{lemma}
     Fix any $T\geq 0$
and assume any $w^\star$ such that
 \begin{align}
\|w^\star-w_0\|^2\geq \max\{\eta T \widehat F(w^\star),\eta \widehat F(w_0)\}.
\end{align}
Pick $\widetilde \rho = 3\rho^\star = 3\|w^\star-w_0\|$ and assume $m$ is large enough such that $\widetilde \rho^2<1/2\widetilde C$ where $\widetilde C= \frac{\beta_L \tilde\rho^{3L}}{\sqrt{m}}$. Then $\rho(t):=\|w_t-w_0\| \le \widetilde \rho$ and $\|w_t-w^\star\| \le \widetilde \rho$ for all $t\in[T]$.
\end{lemma}

\begin{proof}
The statements of the lemma are correct for $t=0$. Now assume for $t\le T-1$ we have $\rho(t) \le \widetilde \rho$ and $\|w_t-w^\star\|\le \widetilde \rho.$ Recall, by \ref{eq:induct}
\bea
\|w_T-w^\star\|^2 &\le \|w_{T-1}-w^\star\|^2 + \eta \Fh(w^\star) + \frac{\eta}{2} \frac{\beta_L\tilde{\rho}^{3L}}{\sqrt{m}} \|w_{T-1}-w^\star\|^2 \max_{w'_{T-1}\in[w_{T-1},w^\star]}\widehat F(w'_{T-1})
\eea
Note that by the assumption of the lemma and the induction assumption it holds that 
$$\max\{\|w_{T-1}-w_0\|,\|w^\star-w_0\| \}\le \widetilde \rho.$$ As a result, for any $w\in[w_{T-1},w^\star]$, it holds that $\|w-w^\star\|\le \tilde \rho$ and thus $\|\nabla \Fh(w)\|\ge \frac{\beta_L\tilde\rho^{3L}}{\sqrt{m}}\Fh(w)$. Moreover, note that as per the lemma's assumption $\|w_{T-1}-w^\star\|^2 \le \widetilde \rho^2\le 1/2\widetilde C$. Thus, the conditions of Proposition \ref{prop:GLQCapp} hold with $\tau=4/3$ and we have 
\begin{align*}
\|w_T-w^\star\|^2 &\le \|w_{T-1}-w^\star\|^2 + \eta \Fh(w^\star) + \frac{2\eta}{3} \frac{\beta_L \tilde\rho^{3L}}{\sqrt{m}} \|w_{T-1}-w^\star\|^2\cdot\max\{\Fh(w_{T-1}),\Fh(w^\star)\}
\end{align*}
Noting again that $\|w_{T-1}-w^\star\|^2\le \widetilde \rho^2\le 1/2\widetilde C$ yields,
\begin{align*}
\|w_T-w^\star\|^2 &\le \|w_{T-1}-w^\star\|^2 + \eta \Fh(w^\star) + \frac{\eta}{3}\max\{\Fh(w_{T-1}),\Fh(w^\star)\}\\
&\le \|w_{T-1}-w^\star\|^2 + \frac{4\eta}{3} \Fh(w^\star) +\frac{\eta}{3} \Fh(w_{T-1})
\end{align*}

Since the conditions of the lemma hold for all $t\in[T-1]$, we can telescope the summation to obtain,
\begin{align*}
\|w_T-w^\star\|^2 &\le \|w_0-w^\star\|^2 + \frac{4\eta}{3} T\Fh(w^\star) + \frac{\eta}{3} \sum_{t=0}^{T-1}\Fh(w_t) \\
&\le \|w_0-w^\star\|^2 + \frac{4\eta}{3} T\Fh(w^\star) + \frac{\eta}{3} \Fh(w_0) +  \frac{\eta}{3} \sum_{t=1}^{T-1}\Fh(w_t)\\
&\le \|w_0-w^\star\|^2 + \frac{4\eta}{3} T\Fh(w^\star) + \frac{\eta}{3} \Fh(w_0) + \frac{\eta}{3} \Big(2 T\Fh(w^\star)+\frac{3\left\|w^\star-w_0\right\|^2}{2\eta }+ \frac{1}{2} \Fh(w_0)\Big)\\
 &= \|w_0-w^\star\|^2 + 2\eta T \Fh(w^\star) + \frac{\eta}{2} \Fh(w_0) + \frac{\|w^\star-w_0\|^2}{2} 
\end{align*}
Therefore by the conditions of the lemma and noting that $A_0 = \|w_0-w^\star\|^2 = {(\rho^\star)}^2$:
\begin{align*}
\|w_T-w^\star\|^2 \le \|w_0-w^\star\|^2 + 2 \|w_0-w^\star\|^2 + \frac{1}{2} \|w^\star-w_0\|^2 +\frac{1}{2} \|w^\star-w_0\|^2  = 4 (\rho^\star)^2
\end{align*}
Thus, $\|w_T-w^\star\|^2\le 9(\rho^\star)^2 = \widetilde\rho^2.$ It remains to prove that $\rho(T):= \|w_T-w_0\|\le \widetilde \rho$. To obtain this, note that by triangle inequality, 
\begin{align*}
\|w_T-w_0\| &\le \|w_T-w^\star\| + \|w^\star - w_0\| \\
&\le 2 \rho^\star +\rho^\star = 3\rho^\star=\widetilde \rho.
\end{align*}
The proof is complete.
\end{proof}

The lemma above shows that controlling $\rho^\star=\|w^\star-w_0\|$ is enough to guarantee that GLQC property (Proposition \ref{prop:GLQCapp}) holds throughout the gradient path. In particular, if the width is large enough such that for $\widetilde\rho = 3\rho^\star$: $
\widetilde \rho^2 \le \frac{1}{2\widetilde C} = \frac{\sqrt{m}}{2\beta_L \tilde\rho^{3L}}$
or equivalently 
$$\sqrt{m}\ge 2\beta_L \tilde \rho^{3L+2}.$$
then the conditions of the lemma hold and the bounds on the training loss are valid.  Note that we need to ensure the conditions of the descent lemma (Lem \ref{lem:des}) hold throughout the GD iterates. This is guaranteed by taking 
\begin{align}\label{eq:descon}
\eta\le \frac{1}{(\widetilde C_1)^2+\tilde C}
\end{align}
where note that $\tilde C_1 = G_0 + \beta_L \frac{(3\|w^\star-w_0\|)^{3L+1}}{\sqrt{m}},$ and $\tilde C = \beta_L \frac{(3\|w^\star-w_0\|)^{3L}}{\sqrt{m}}$. As per Lemma \ref{lem:gradients_and_hessians_bounds_general}, this condition ensures that step-size is smaller than the inverse of objective's smoothness parameter throughout all iterates since $\|w_t-w_0\|\le 3\|w^\star -w_0\|$ for all $t\in[T]$.

This completes the proof of Theorem \ref{thm:b1}.

\subsection{Proof for Generalization loss}
\label{sec:proofthm2}
\begin{theorem}[Generalization loss of GD in deep nets]\label{thm:b2}
    Let the descent lemma hold. Fix a $T$ and assume the target weights vector $w^\star\in\R^p$ that separates the data distribution such that 
\bea
\rho^\star\geq \max\left\{\sqrt{\eta T \widehat F(w^\star)},\sqrt{\eta \widehat F(w_0)}\right\}.
\eea
where $\rho^\star:=\|w^\star-w_0\|$. Moreover, assume the width $m$ is large enough such that it satisfies,
\bea
\sqrt{m}\ge 4 \beta_L\bar \rho^{3L+2}
\eea
where $\bar \rho  = 6\rho^\star$.
Then the expected generalization error satisfies with high probability,
\bea
\E_{\mathcal{S}}\Big[{F}(w_T)-\widehat F(w_T)\Big] \le 9\frac{ (G_0+1/4)^2  {\rho^\star}^2}{n},
\eea 
where $G_0$ is the Lipschitz parameter of network at initialization  i.e., $\|\nabla \Phi(w_0,\cdot)\|\le G_0.$ 
%In particular, after $T=n$ iterations, the test error reaches its optimal value:
%\bea
%\E_{\mathcal{S}}\Big[{F}(w_T)\Big] = O\left(\frac{{\rho^\star}^2G_0^2}{n}+ \frac{  {\rho^\star}^2}{\eta n}\right)
%\eea
\end{theorem}
As per the last theorem, we know if $m$ is large enough based on $\rho^\star$, then the hessian during the iterates of GD satisfies $\|\nabla^2 \Phi(w,x)\|\le  C_2 $ and thus as per Lemma \ref{lem:gradients_and_hessians_bounds_general} we have,

\begin{enumerate}
  
\item $\|\nabla \Fh(w)\|_F\le C_1 \, \Fh(w).$

\item $\|\nabla^2 \Fh(w)\|_2 \leq  ( C_2+ (C_1)^2) \Fh(w).$

\item $\lambda_\min\left(\nabla^2 \widehat F(w)\right) \geq -  C_2 \Fh(w)$.

\end{enumerate}
throughout this section and with a slight abuse of notation, we define $C_1 = G_0 + \beta_L (3\|w-w_0\|)^{3L+1}/\sqrt{m},$ $C_2 = \beta_L (3\|w-w_0\|)^{3L}/\sqrt{m}$ and $G_0$ is the Lipschitz parameter of the network at initialization i.e., $\|\nabla \Phi(w_0,\cdot)\|\le G_0.$ Recall that in obtaining these quantities we used the guarantees of Thm \ref{thm:train-test} that along the gradient descent path $\max\{\|w_t-w_0\|,\|w_t-w^\star\|\}\le 3\|w^\star-w_0\|$
%where the parameters $\widetilde C_1,\widetilde C_2$ are independent of iteration numbder $t.$ We expect the parameter $\widetilde C_2$ to scale as $1/ \sqrt{m}$ and $\widetilde C_1$ to scale as $O_m(1)$ in the NTK regime, although the dependencies on $L$ might be exponential. 

We derive the generalization bound through leave-one-out error estimates. In particular, let $w_t^{\neg i}$ denote the output of $t$-th iteration of GD on the leave-one-out loss $F^\negi(\cdot):\R^p\rightarrow\R$ defined as $F^\negi (w):=$ $\frac{1}{n} \sum_{j \neq i} \widehat{F}_j(w)$. Then, a simple calculation based on the Lipschitz property of the objective relates the generalization loss to the leave-one-out error as stated in the following lemma. 

\begin{lemma}[Lem. 6 in \cite{schliserman2022stability}, Thm. 2 in \cite{lei2020fine}]\label{lem:sk22}
   Suppose each sample loss $f(\cdot,z)$ is $G$-Lipschitz  for almost surely all data points $z\sim\mathcal{D}$. Then, the following relation holds between expected generalization loss and model stability at any iterate $T$,
   \begin{align}
       \E_{\mathcal{S}}\Big[{F}(w_T)-\widehat F(w_T)\Big] \leq  2G\cdot \;\E_{\mathcal{S}}\Big[\frac{1}{n}\sum_{i=1}^n\|w_T-w_T^{\neg i}\|\Big].
   \end{align}
\end{lemma}
We will use the fact that under the descent lemma $\eta<1/( C_2+C_1^2)$, based on \cite[Lem. 9]{taheri2024generalization}:

\begin{align}\label{eq:expansion}
& \left\|\left(w-\eta \nabla \widehat{F}(w)\right)-\left(w^{\prime}-\eta \nabla \widehat{F}\left(w^{\prime}\right)\right)\right\| \leq \max _{\tilde w \in[w,w']} (1+ \eta C_2\Fh(\tilde w))\left\|w-w^{\prime}\right\| .
\end{align}

Then, we can bound the leave-one-out error as follows:
\begin{align*}
\left\|w_{t+1}-w_{t+1}^{\neg i}\right\| & \leq\left\|\left(w_t-\eta \nabla \widehat{F}^{\neg i}\left(w_t\right)\right)-\left(w_t^{\neg i}-\eta \nabla \widehat{F}^{\neg i}\left(w_t^{\neg i}\right)\right)\right\|+\frac{\eta}{n}\left\|\nabla \widehat{F}_i\left(w_t\right)\right\| \\
& \leq\left\|\left(w_t-\eta \nabla \widehat{F}^{\neg i}\left(w_t\right)\right)-\left(w_t^{\neg i}-\eta \nabla \widehat{F}^{\neg i}\left(w_t^{\neg i}\right)\right)\right\|+\frac{\eta  C_1}{n} \widehat{F}_i\left(w_t\right) \\
& \leq\left(1+\eta  C_2 \max _{w_t ' \in[w_t,w_t^\negi]} \widehat{F}^{\neg i}\left(w'_t\right)\right)\left\|w_t-w_t^{\neg i}\right\|+\frac{\eta C_1}{n} \widehat{F}_i\left(w_t\right),
\end{align*}

By the aforementioned property from Proposition \ref{prop:GLQCapp}, we have:
\begin{align*}
\max_{w_t ' \in[w_t,w_t^\negi]} \widehat{F}^{\neg i}\left(w'_t\right) &\le \frac{4}{3} \max\{ \Fh^{\negi}(w_t),\Fh^{\negi}(w_t^\negi) \}\\
&\le \frac{4}{3} \max\{ \Fh(w_t),\Fh^{\negi}(w_t^\negi) \}\\
&\le \frac{4}{3} (\Fh(w_t) + \Fh^{\negi}(w_t^\negi))
\end{align*}
Note that for $w_{t}^'$ we have 
\bea
\|w_t'-w_0\| \le \alpha \|w_t-w_0\| + (1-\alpha)\|w_t^\negi-w_0\| \le \widetilde \rho
\eea
therefore the self-bounded weak convexity holds with $\kappa= C_2$.
Moreover, $$\|w_t-w_t^\negi\|\le \|w_t-w_0\| + \|w_t^\negi-w_0\|\le 3 \rho^\star + 3\rho^\star = 6\rho^\star = 2\widetilde \rho$$

Denote $\bar \rho:=6\rho^\star (=2\widetilde\rho)$. 
If $m$ is large enough such that $\bar\rho^2<1/2\bar C_2$ where recall $\bar C_2= \frac{\beta_L \bar\rho^{3L}}{\sqrt{m}},$
then the GLQC property holds with $\tau = 4/3$ and we have,
\begin{align*}
\left\|w_{t+1}-w_{t+1}^{\neg i}\right\| &\le \left(1+\frac{4}{3}\eta\bar C_2 (\Fh(w_t) + \Fh^{\negi}(w_t^\negi))\right)+\frac{\eta C_1}{n} \widehat{F}_i\left(w_t\right)\\
&\le \exp(\frac{4}{3}\eta\bar C_2 (\Fh(w_t) + \Fh^{\negi}(w_t^\negi)))+\frac{\eta C_1}{n} \widehat{F}_i\left(w_t\right)\\
&\le \frac{\eta  C_1}{n} \sum_{k=0}^t \exp\left(\sum_{\tau=1}^t \frac{4}{3}\eta\bar C_2 (\Fh(w_\tau) + \Fh^{\negi}(w_\tau^\negi))\right) \Fh_i(w_k)
\end{align*}
Fix a $T$ and assume $\bar C_2$ is small enough such that 
\bea\label{eq:.05}
\frac{4}{3}\eta\bar C_2\sum_{\tau=1}^T  (\Fh(w_\tau) + \Fh^{\negi}(w_\tau^\negi))\le \frac{16}{3}\bar C_2(\eta T\Fh(w^\star)+{\rho^\star}^2) \le 0.05
\eea
then,
\begin{align*}
\left\|w_{T}-w_{T}^{\neg i}\right\| &\le \frac{\eta C_1 \exp(0.05)}{n}\sum_{t=0}^T \Fh_i(w_t)\\
&\le 1.1\frac{\eta C_1 }{n}\sum_{t=0}^T \Fh_i(w_t).
\end{align*}
The choice of 0.05 in Eq. \ref{eq:.05} is to ensure that the resulting constant in the generalization bound is close to 1.  By averaging over $i\in[n]$ and by  our training loss guarantees from Theorem \ref{thm:train-test}:
\begin{align*}
\frac{1}{n} \sum_{i=1}^n \left\|w_{T}-w_{T}^{\neg i}\right\| &\le 1.1\frac{\eta C_1}{n}\sum_{t=0}^T \Fh(w_t)\\
&\le 2.2\frac{C_1}{n} (\eta T \Fh(w^\star) + {\rho^\star}^2)
\end{align*}

based on Lemma \ref{lem:sk22} and noting that $G= C_1$, we find that
\bea
\E_{\mathcal{S}}\Big[{F}(w_T)-\widehat F(w_T)\Big] \le \frac{4.4 }{n}\, \E_{\mathcal{S}} \left[ C_1^2 \eta T \Fh(w^\star) + C_1^2{\rho^\star}^2 \right] 
\eea

Recall from Lemma \ref{lem:gradients_and_hessians_bounds_general} that $C_1 \le  G_0 + \beta_L\|w_t-w_0\|^{3L+1} / {\sqrt{m}} $ for all $t\in[T]$. As we require $\sqrt{m} \ge 4\beta_L (6\|w^\star-w_0\|^{3L+2})$ and from convergence guarantees it holds that $\|w_t-w_0\| \le 3\|w^\star-w_0\|$ we conclude that 
$$C_1 \le G_0 + \frac{1}{2^{3L+2}\|w^\star-w_0\|} \le G_0 + 1/4,$$
where in the last step we assumed without loss of generality that $\|w^\star-w_0\|\ge 1.$ This completes the proof for the generalization error in Theorem \ref{thm:b2}. The proof of Theorem \ref{thm:train-test} is complete. 

\section{Proof of Corollary \ref{cor:NTK}}
\begin{corollary}[Restatement of Corollary \ref{cor:NTK}]\label{cor:NTKapp}
    Assume the tangent kernel of the model at initialization separates the data with margin $\gamma$ i.e., for a unit-norm $w\in\R^p$ it holds for all $i \in [n]:$
$y_i\left\langle\nabla \Phi\left(w_0, x_i\right), w\right\rangle \geq \gamma$. Define constant $B>0$ that bounds the model's output at initialization i.e., $\forall i\in[n]: |\Phi(w_0,x_i)| < B$. Assume the width is large enough such that $m\ge \beta_L^2 (\frac{2B+\log(1/\eps)}{\gamma})^{6L+4}$. Then there exists $w^\star\in\R^p$ such that $F(w^\star)\le \eps$ and $w^\star$ lies at a  Euclidean distance of order $O(1/\gamma)$ from initialization where $\|w^\star-w_0\|\le \frac{1}{\gamma}(2B+\log(1/\eps))$.
\end{corollary}
\begin{proof}[Proof of Corollary \ref{cor:NTK}]
By Taylor there exists $w^{\prime} \in\left[w^\star, w_0\right]$ such that,
$$
y_i \Phi\left(w^\star, x_i\right)=y_i \Phi\left(w_0, x_i\right)+y_i\left\langle\nabla \Phi\left(w_0, x_i\right), w^\star-w_0\right\rangle+\frac{1}{2} y_i\left\langle w^\star-w_0, \nabla^2 \Phi\left(w^{\prime}, x_i\right)\left(w^\star-w_0\right)\right\rangle
$$. 

Pick $w^\star=w_0+\frac{w}{\gamma}(2 B+\log (1 / \varepsilon))$. Since $\left\|w\right\|=1$, we  derive the desired for $\left\|w^\star-w_0\right\|$. we have $\left\|\nabla^2 \Phi\left(w^{\prime}, x_i\right)\right\| \leq O(\frac{\rho^{3L}}{\sqrt{m}})$. Choosing $\rho =\|w^\star-w_0\| = \frac{1}{\gamma}(2B+\log(1/\eps))$ and noting the given condition on $m$ implies $m \ge \beta_L^2 \rho^{6L+4}$ we obtain $
\frac{\beta_L\rho^{3L+2}}{\sqrt{m}}\le 1$.
Therefore with the given assumption on  $m\ge \frac{\beta_L^2} {\gamma^{6L+4}}(2B+\log(1/\eps))^{6L+4}$ leads to $\frac{1}{2}\left\|\nabla^2 \Phi\left(w^{\prime}, x_i\right)\right\|\left\|w^\star-w_0\right\|^2 \le 1.$ Then,
\begin{align*}
y_i \Phi\left(w^\star, x_i\right) & \geq-\left|y_i \Phi\left(w_0, x_i\right)\right|+y_i\left\langle\nabla_1 \Phi\left(w_0, x_i\right), w^\star-w_0\right\rangle-\frac{1}{2}\left\|\nabla_1^2 \Phi\left(w^{\prime}, x_i\right)\right\|\left\|w^\star-w_0\right\|^2 \\
& \geq-B+2 B+\log (1 / \varepsilon)-B\\
& = \log (1 / \varepsilon)
\end{align*}
where we assumed without loss of generality that $B\ge 1.$
After applying the logistic loss function on $y_i\Phi(w^\star,x_i)$, the inequality above gives $\widehat{F}(w) \leq \varepsilon$. This completes the proof.
\end{proof}

\section{Consistency of GD in deep nets: Proof of Theorem \ref{thm:generalization}}\label{sec:proofthm3}
\begin{theorem}[Consistency]
Assume the width of the network satisfies $m\ge\beta_L^2 n^{3L+3}$ and the step-size satisfies $\eta<1/L_F$. Then with high probability the expected test loss at iteration $T\ge \sqrt{n}$ is bounded as:
\bea\label{eq:thm3bound}
\E_{\mathcal{S}}\left[F(w_T)\right] \le \left( 1+ \frac{4}{\sqrt{n}}\right) \left( F(w^\star) + \frac{{\rho^\star}^2}{\eta \sqrt{n}} + \frac{{\rho^\star}^2}{n \sqrt{n}}  + \frac{1}{\sqrt{n}}\right)
\eea
where $w^\star$ is the minimizer of the population loss i.e., $w^\star = \arg\min_{w\in\R^p} F(w)$. 
\end{theorem}
Recall that based on \ref{eq:expansion}
\begin{equation}
\left\|(w-\eta \nabla \widehat{F}(w))-\left(w^{\prime}-\eta \nabla \widehat{F}\left(w^{\prime}\right)\right)\right\| \leq 1+\eta \max_\alpha C_2\left\|w-w^{\prime}\right\|,
\end{equation}
where $ C_2 :=  O(\frac{\beta_L \rho^{3L}}{\sqrt{m}})$ and $\rho:= \|w-w_0\|$

In particular, note that for GD, we can obtain for $\|w-w_0\|$:
\begin{align}
\|w_T - w_0\| &\le \eta\sum_{t=0}^{T-1} \|\nabla \Fh(w_t) \| \nn\\
&\le \eta \sum_{t=0}^{T-1} C_1\nn\\
&= \eta T C_1,\label{eq:diffbound}
\end{align}
where the first step follows by update rule of GD and in the second step we used the Lipschitz property of objective during GD iterates that ensures $\|\nabla \hf (w_t)\|$ remains bounded by $C_1$ for all $t\in[T]$.
Thus, we can thus show by induction that $\|w_T-w_0\| \le \eta T C_1$. Then for $m=\Omega(T^2)$ and $w_T^\negi,w_T:$
\bea
\left\|(w_T-\eta \nabla \widehat{F}(w_T))-\left(w^\negi_T-\eta \nabla \widehat{F}\left(w_T^\negi\right)\right)\right\| \leq (1+\eta C_2)\left\|w_T-w_T^\negi\right\|,
\eea

where $ C_2 = \beta_L \tilde \rho^{3L}/\sqrt{m}$ is the smoothness parameter of the objective throughout the GD updates  $w_t$ for all $t$. Then,
\begin{align*}
\left\|w_{t+1}-w_{t+1}^{\neg i}\right\| & \leq\left\|\left(w_t-\eta \nabla \widehat{F}^{\neg i}\left(w_t\right)\right)-\left(w_t^{\neg i}-\eta \nabla \widehat{F}^{\neg i}\left(w_t^{\neg i}\right)\right)\right\|+\frac{\eta}{n}\left\|\nabla \widehat{F}_i\left(w_t\right)\right\| \\
& \leq\left\|\left(w_t-\eta \nabla \widehat{F}^{\neg i}\left(w_t\right)\right)-\left(w_t^{\neg i}-\eta \nabla \widehat{F}^{\neg i}\left(w_t^{\neg i}\right)\right)\right\|+\frac{\eta C_1}{n} \widehat{F}_i\left(w_t\right) \\
& \leq\left(1+\eta C_2 \right)\left\|w_t-w_t^{\neg i}\right\|+\frac{\eta C_1}{n} \widehat{F}_i\left(w_t\right),
\end{align*}

Recall that in the above $\Fh_i(\cdot)$ refers to the objective at the $i$th data sample and thus has the same Lipschits constant as $\Fh(w) = \frac{1}{n} \sum_{i=1}^n \Fh_i(w)$. 

Noting that $1+x\le e^x$ for $x\ge0$, we can proceed with the above inequality as follows by unrolling the iterates of $\|w_t-w_t^\negi\|$ and the fact that $w_0=w_0^\negi$:
\begin{align*}
\left\|w_{t+1}-w_{t+1}^{\neg i}\right\| &\le \exp(\eta C_2) \|w_t-w_t^\negi\| + \frac{\eta  C_1}{n} \widehat{F}_i\left(w_t\right)\\
&\le \frac{\eta  C_1}{n} \sum_{k=0}^t \exp(\eta  C_2 t) \Fh_i(w_k).
\end{align*}

Recall that, $ C_2 = \beta_L \tilde\rho^{3L}/\sqrt{m}$, and $\tilde \rho = \max\{\|w_t-w_0\|,\|w_{t-1}-w_0\|,\cdots,\|w_1-w_0\|\}$.
Based on Eq. \ref{eq:diffbound}, $\|w_t-w_0\|\le \eta t C_1$ which implies $\tilde\rho \le \eta t C_1.$ As a result, for some fixed $T$:
\bea\nn
\|w_T - w_T^\negi\|\le \frac{\eta C_1}{n} \sum_{t=0}^{T-1} \exp \left(\frac{\eta \beta_L T}{\sqrt{m}} (\eta T C_1)^{3L}\right) \Fh_i(w_t).
\eea
Thus for $m$ larger enough such that $\sqrt{m} \ge 2 \beta_L (\eta C_1 T)^{3L+1}$, the exponential term is smaller than $2$ and we conclude that,
\bea
\left\|w_T-w_T^{\neg i}\right\| &\le \frac{2\eta  C_1}{n} \sum_{t=0}^{T-1}\Fh_i(w_t).
\eea
Therefore, by averaging over $i\in[n]$ we derive the on-average leave-one-out:
\bea
\frac{1}{n} \sum_i \left\|w_T-w_T^{\neg i}\right\|  \le  \frac{2\eta C_1}{n} \sum_{t=0}^{T-1}\Fh(w_t)
\eea
The only remaining part is characterising the right hand side of the above inequality. This is done as follows.
\subsection*{Proof of optimization error}
Define $\rho(t):= \|w_t-w_0\|$ and $\rho^* = \|w^\star-w_0\|$. if we define $\widetilde \rho > \max\{\rho(t),\rho^\star\}$ for all $t\in[T]$ then for any $w\in[w_t,w^\star]$ it holds that $\|w-w_0\|\le \widetilde \rho$ and thus based on Lemma \ref{lem:gradients_and_hessians_bounds_general}

\bea
\lambda_{\min}(\nabla ^2 \Fh(w)) \ge - \widetilde C \Fh(w),
\eea
where $\widetilde C :=  \frac{\beta_L \tilde\rho^{3L}}{\sqrt{m}}$.
Therefore by Eq. \ref{eq:trainloss_gen},
\begin{align*}
\Fh(w_T) \le \frac{1}{T} \sum_{t=1}^T \widehat{F}\left(w_t\right) &\leq \widehat{F}(w^\star)+\frac{\left\|w^\star-w_0\right\|^2}{\eta T}+\widetilde C \frac{1}{2T} \sum_{t=0}^{T-1} \max _{w'_t\in[w_t,w^\star]} \widehat{F}\left(w_t'\right)\left\|w^\star-w_t\right\|^2,\\
&\le \widehat{F}(w^\star)+\frac{\left\|w^\star-w_0\right\|^2}{\eta T}+\widetilde C \frac{1}{2T} \sum_{t=0}^{T-1} \left\|w^\star-w_t\right\|^2\\
&\le \widehat{F}(w^\star)+\frac{\left\|w^\star-w_0\right\|^2}{\eta T}+\widetilde C  \left\|w^\star-w_0\right\|^2 + \widetilde C \frac{1}{T} \sum_{t=0}^{T-1} \|w_t-w_0\|^2.
\end{align*}
For the last term, we have a simple bound:
\bea\nn
\|w_t-w_0\| = \eta \|\sum_{k=0}^{t-1} \nabla \Fh(w_k)\| \le \eta \sum_{k=0}^{t-1} \|\nabla \Fh(w_k)\|\le \eta t C_1 
\eea
where $C_1$ denote the lipschitz parameter bound along the GD iterates and the last inequality follows by descent lemma condition. 

Then,
\bea\nn
\Fh(w_T)\le \widehat{F}(w^\star)+\frac{\left\|w^\star-w_0\right\|^2}{\eta T}+\widetilde C  \left\|w^\star-w_0\right\|^2 + \eta^2 \widetilde C C_1^2 T^2.
\eea
Pick large enough width such that $m \ge \beta_L^2 n^{3L+3} (\eta C_1)^{6L}$ and choose $T=\sqrt{n}$. Then $\tilde \rho<(\eta C_1 n)^{3L/2}$ and 
\begin{align*}
\Fh(w_T) &\le \widehat{F}(w^\star)+\frac{{\rho^\star}^2}{\eta \sqrt{n}}+  \frac{\beta_L \tilde \rho^{3L}}{\sqrt{m}} ({\rho^\star}^2+ \eta^2 C_1^2 n )\\
& \le  \widehat{F}(w^\star)+\frac{{\rho^\star}^2}{\eta \sqrt{n}} + \frac{{\rho^\star}^2+ \eta^2 C_1^2 n }{n\sqrt{n}}
\end{align*}
In particular, in the above one can choose $w^\star:=\arg\min _{w\in\R^p}F(w)$. Moreover, from the generalization guarantees,
\begin{align*}
\E_{\mathcal{S}}\left[F(w_T)\right] &\le \E_{\mathcal{S}}[\Fh(w_T)] + \frac{4\eta C_1^2}{n} \sum_{t=1}^{T-1} \E_{\mathcal{S}}[\Fh(w_t)]\\
&\le \left(1+\frac{4\eta C_1^2}{\sqrt{n}}\right) \left(\E_{\mathcal{S}}[\Fh(w^\star)] + \frac{{\rho^\star}^2}{\eta \sqrt{n}}+\frac{{\rho^\star}^2+ \eta^2 C_1^2 n }{n\sqrt{n}}\right) \\
&= \left(1+\frac{4\eta C_1^2}{\sqrt{n}}\right) \left(F(w^\star) + \frac{{\rho^\star}^2}{\eta \sqrt{n}}+\frac{{\rho^\star}^2+ \eta^2 C_1^2 n }{n\sqrt{n}}\right).
\end{align*}

Note that $\eta\le \frac{1}{C_1^2+ C_2}$ by the descent lemma condition (Lemma \ref{lem:des}). Therefore 
$\eta\le 1/C_1^2.$ Without loss of generality we can assume $C_1\ge 1$. This simplifies the above inequality as follows:
\begin{align*}
\E_{\mathcal{S}}\left[F(w_T)\right] &\le \left( 1+ \frac{4}{\sqrt{n}}\right) \left( F(w^\star) + \frac{{\rho^\star}^2}{\eta \sqrt{n}} + \frac{{\rho^\star}^2}{n \sqrt{n}}  + \frac{1}{\sqrt{n}}\right)\\
&= F(w^\star)+ O\left(\frac{F(w^\star) + {\rho^\star}^2}{\sqrt{n}}\right).
\end{align*}
This completes the proof of Theorem \ref{thm:generalization}.

%%%%%%%%%%%%%%%%%%%%%%% XOR THEOREM PROOF %%%%%%%%%%%%%%%

\section{Proof of Theorem \ref{thm:xor}}\label{sec:proof_xor}

We recall our setup and notation for Theorem \ref{thm:xor}. We consider in this section the one-hidden layer network with quadratic activation where $x\in\R^d,w_i\in \R^d$:
\begin{align*}
    \Phi(w,x) = \frac{1}{2m}\sum_{i=1}^m a_i (x^\top w_i)^2
\end{align*}
in the above, $a_i = \pm 1$ are fixed and half of the neurons have $a_i=1$. For initialization of first layer weights $w_i \simiid N(0,\frac{I_d}{d})$ that is each $w_{ij}\simiid N(0,\frac{1}{d})$. The data points $x$ are uniformly drawn from the Rademacher distribution (of size $2^d$) that is $x(k)=\pm 1$ i.i.d and $w.p. \; 1/2$ and the labels $y=x(1)\cdot x(2).$ Consider the linear loss where $f(u)=-u$ and $\widehat F(w):= \frac{1}{n}\sum_{j=1}^n f(y_j \Phi(x_j,w))$. Consider stochastic gradient descent with the update rule as follows for each $w_i$:
\begin{align*}
w_i^{t+1} &= w_i^t - \eta \nabla \widehat F_i(w^t) 
\\&= w_i^{t} - \frac{\eta a_i}{nm}  \sum_{j=1}^n (x_j^\top w_i^t) x_j y_j f'(y_j\Phi(x_j,w)) \\
&= w_i^{t} + \frac{\eta a_i}{nm}  \sum_{j=1}^n (x_j^\top w_i^t) x_j y_j
\end{align*}
Note that in the above the superscripts denote the iteration number and the subscript $i$ indicates the vector $w_i\in\R^d$ is associated with neuron $i.$
Fixing the initialization, define $\E_x[w_i^t]$ as the outcome of GD after $t$ iterations on \emph{population gradient}. Formally, $\E_x[w_i^t]$ is defined recursively as follows:
\begin{align*}
    \E_x[w_i^{t+1}] := \E_x[w_i^t] - \eta \nabla  F_i(\E_x[w^t]),
\end{align*}
where $F(w) := \E_x[\widehat F(w)]$ is the test loss and define $\E_x[w^0]:=w^0$. Based on previous derivations the update rule for the expected weights $\E_x[w_i^t]$ obeys the following:

\begin{align*}
\E_x[w_i^{t+1}] = \E_x[w_i^t] + \frac{\eta }{m} a_i\E_x \left[\langle x,\E_x[w_i^t]\rangle x y  \right].
\end{align*}
for $t=0:$ 

\begin{align*}
\E_x[w_i^{1}] &= w_i^0 + \frac{\eta}{m}  a_i\E_x \left[\langle x,w_i^0\rangle x y\right] .
\end{align*}
For $t=1:$
\begin{align*}
\E_x[w_i^2] =  \E_x[w_i^1] + \frac{\eta }{m}a_i \E_x \left[\langle x,\E_x[w_i^1]\rangle x y \right].
\end{align*}
and similarly $\E_x[w_i^t]$ is defined for all iterations $t$. 
\subsection{Calculating $\E_x[w_i^1]$}
For the expected gradient $g_i\in\R^d$ we have by noting $y=x(1)x(2)$:
\begin{align*}
g_i = -\E_x[ \langle x,w_i^0 \rangle\cdot x \cdot x(1)x(2)]
\end{align*}
for $g_i^3$ (we drop $i$ and 0 for simplicity and denote the $k'$th element of $W_i^0$ with $W^k$ ):
\begin{align*}
-g_i^3 &= \E_x [(x(3) w^3 +\sum_{k\neq 3} x(k) w^k)\cdot  x(3) \cdot x(1)x(2)]\\
&= \E_x[w^3 \cdot x(1)x(2) + x(3)x(1)x(2)\sum_{k\neq 3} x(k) w^k]\\
& =  \E_x[w^3 \cdot x(1)x(2)] + \E_x[x(3)x(1)x(2)\sum_{k\neq 3} x(k) w^k]\\
&= 0 + \E_x[x(3) \sum_{k\neq 3} x(k) w^k]\, \E_x[x(1)x(2)]\\
&=0.
\end{align*}
Therefore for any $k\neq 1,2$ we have $g_i^k=0.$

for $k=1$ :
\begin{align*}
-g_i^1 &= \E_x [ \langle x, w \rangle\cdot x(1) \cdot x(1)x(2)] \\
&= \E_x[w^1 x(1)x(2) + x(1)x(2) x(1)\sum_{k\neq 1} w^k x(k)]\\
& = \E_x[{w^1}x(1)x(2)] + \E_x[x(1)x(2) x(1) w^2 x(2)]+\E_x[x(1)x(2) x(1) \sum_{k\neq 1,2}w^k x(k)]\\
& = {w^1} \E_x[x(1)x(2)]  + w^2 \E_x[x(1)x(2) x(1)x(2)] + \E_x[x(1)x(2) x(1) ]\cdot \E_x[\sum_{k\neq 1,2}w^k x(k)]\\
&= w^2 \E_x[(x(1)x(2))^2] \\
&= w^2.
\end{align*}

Similarly, we derive that
\begin{align*}
-g_i^2 = w_i^0(1)
\end{align*}
Thus, the expected weight after one step is 
\begin{align*}
\E_x[w_i^1] &= w_i^0 + \frac{\eta a_i}{m} [w_i^0(2),w_i^0(1),0,\cdots,0]\\
&=\Big[w_i^0 (1)+ \frac{\eta a_i}{m} w_i^0(2), w_i^0(2)+\frac{\eta a_i}{m} w_i^0(1), w_i^0(3), \cdots,w_i^0(d) \Big]
\end{align*}

\subsection{Calculating $\E_x[w_i^2]$}

For $k=1:$

\begin{align*}
\E_x[w_i^2(1)] &=  \E_x[w_i^1(1)]+ \frac{\eta a_i}{m} \E_x[\langle x,\E_x[w_i^1]\rangle x(2)]\\
& = \E_x[w_i^1(1)]+ \frac{\eta a_i}{m} \E_x[w_i^1 (2)]\\
&=  w_i^0 (1)+ \frac{\eta a_i}{m} w_i^0(2)+ \frac{\eta a_i}{m} (w_i^0(2)+ \frac{\eta a_i}{m} w_i^0(1))
\end{align*}
Similarly we obtain for $k=2:$

\begin{align*}
\E_x[w_i^2(2)] &=  \E_x[w_i^1(2)]+ \frac{\eta a_i}{m} \E_x[\langle x,\E_x[w_i^1]\rangle x(1)]\\
& = \E_x[w_i^1(2)]+ \frac{\eta a_i}{m} \E_x[w_i^1 (1)]\\
&=  w_i^0 (2)+ \frac{\eta a_i}{m} w_i^0(1)+ \frac{\eta a_i}{m} (w_i^0(1)+ \frac{\eta a_i}{m} w_i^0(2))
\end{align*}
for $k\neq 1,2$:
\begin{align*}
    \E_x[w_i^2(k)] &=  \E_x[w_i^1(k)]+ \frac{\eta a_i}{m} \E_x[\langle x,\E_x[w_i^1]\rangle x(k)x(1)x(2)]\\
    &=\E_x[w_i^1(k)]
\end{align*}
\subsection{Calculating $\E_x[w_i^t]$}

In general, we observe the following pattern for arbitrary $t$:
for $k\neq 1,2:$
\begin{align*}
\E_x[w_i^{t+1}(k)] = \E_x[w_i^{t}(k)]
\end{align*}
and for $k=1,2$:
\begin{align*}
\E_x[w_i^{t+1} (1)] &= \E_x[w_i^{t}(1)] + \frac{\eta a_i}{m}\E_x[w_i^t(2)]\\
\E_x[w_i^{t+1} (2)] &= \E_x[w_i^{t}(2)] + \frac{\eta a_i}{m}\E_x[w_i^t(1)]
\end{align*}
Therefore for general $t$, defining $\gamma = \eta a_i/m$, $\alpha(t):= \E_x[w_i^t(1)]$ and $\beta(t):=\E_x[w_i^t(2)]$ we have the following equations:
\begin{align*}
    \alpha(t+1) &= \alpha(t) + \gamma \beta(t)\\
    \beta(t+1) &= \beta(t)+ \gamma \alpha(t).
\end{align*}
It can be verified that the solution to the above equations take the following form:
\begin{align*}
\alpha(t) &= \left(\sum_{r\in[t],r: even} \binom{t}{r} \gamma^r\right) \alpha(0)  + \left(\sum_{r\in[t],r: odd} \binom{t}{r} \gamma^r\right) \beta(0) \\
\beta(t) &= \left(\sum_{r\in[t],r: even} \binom{t}{r} \gamma^r\right) \beta(0)  + \left(\sum_{r\in[t],r: odd} \binom{t}{r} \gamma^r\right) \alpha(0)
\end{align*}
Therefore with replacement:
\begin{align*}
\E_x[w_i^t(1)] &= \left(\sum_{r\in[t],r: even} \binom{t}{r} (\frac{\eta a_i}{m})^r\right) w_i^0(1)  + \left(\sum_{r\in[t],r: odd} \binom{t}{r} (\frac{\eta a_i}{m})^r\right) w_i^0(2) \\
\E_x[w_i^t(2)] &= \left(\sum_{r\in[t],r: even} \binom{t}{r} (\frac{\eta a_i}{m})^r\right) w_i^0(2)  + \left(\sum_{r\in[t],r: odd} \binom{t}{r} (\frac{\eta a_i}{m})^r\right) w_i^0(1)
\end{align*}

\subsection{Empirical Gradient for multiple GD steps}
The results above hold when GD is applied on the entire data distribution. Let us consider the case of SGD on $n$ data points (on the loss $F(\cdot)$) for each iteration and bound the resulting noise. 

Recall that 

\begin{align*}
w^{1} &= w^0 - \eta \nabla F(w^0), \\
\E_x[w^1] &= w^0 - \eta \E_x[\nabla F(w^0)],\\
w^1 &= \E_x[w^1] - \eta (\nabla F(w^0)-\E_x[\nabla F(w^0)]).
\end{align*}
Moreover, by how we defined $\E_x[w^t]$ for $t \ge 2$, we have
\begin{align*}
w^{t+1} &= w^t - \eta \nabla F(w^t),\\
\E_x[w^{t+1}] &= \E_x[w^{t}] - \eta \E_x[\nabla F(\E_x[w^t])], \\
w^{t+1} &= \E_x[w^t] + (w^t - \E_x[w^t]) - \eta (\nabla F(w^t)- \E_x[\nabla F(\E_x[w^t])]) - \eta \E_x[\nabla F(\E_x[w^t])].
\end{align*}
Overall, from the last two equations we have:
\begin{align*}
w^{t+1} - \E_x[w^{t+1}] &= -\eta \sum_{\tau=0}^t \nabla F(w^\tau) - \E_x\left[\nabla F(\E_x[w^\tau])\right]\\
& = -\eta \sum_{\tau=0}^t \underbrace{\nabla F(w^\tau) - \nabla F(\E_x[w^\tau])}_{\text{Term I}} + \underbrace{\nabla F(\E_x[w^\tau])- \E_x[\nabla F(\E_x[w^\tau])]}_{\text{Term II}}.
\end{align*}
We bound each term separately.
\subsection{Term II:}
Starting with the second term, recall that:
\begin{align*}
-\nabla F_i (\E_x[w^t]) = \frac{a_i}{nm}\sum_{j=1}^n \langle x_j,\E_x[w^t_i]\rangle x_jy_j
\end{align*}
\paragraph{Case 1:} For $k\neq 1,2$ :
\begin{align*}
-\nabla  F_i^k (\E_x[w^t]) &= \frac{a_i}{nm} \sum_{j=1}^n \langle x_j,\E_x[w_i^t]\rangle x_j(k)x_j(1)x_j(2)\\
&= \frac{a_i}{nm}  \left\langle \sum_{j=1}^n x_j\cdot x_j(k)\cdot x_j(1)\cdot x_j(2),\E_x[w_i^t]\right \rangle 
\end{align*}

Denoting $v_j : = x_j\cdot x_j(k)\cdot x_j(1)\cdot x_j(2)$ note that $v_j(r)$ are mutually independent for all $r\in d$ and $j\in[n]$. Moreover each $v_j(r)$ is distributed according to a Rademacher distribution and $Var(v_j(r))=1$. 

Therefore, recalling the expression for $\E_x[w^t_i] = [\alpha_2^t w_i^0(1)+\alpha_1^t w_i^0(2),\alpha_1^t w_i^0(1)+ \alpha_2^t w_i^0(2),w_i^0(3),\cdots,w_i^0(d)]$
where $\alpha_2^t:=\sum_{r\in[t],r: even} \binom{t}{r} (\frac{\eta a_i}{m})^r$ and $\alpha_1^t := \sum_{r\in[t],r: odd} \binom{t}{r} (\frac{\eta a_i}{m})^r$
by expanding the summation above:
\begin{align*}
-\nabla  F_i^k (\E_x[w^t]) &= \frac{a_i}{nm} \Big( (\alpha_2^t \sum_j v_j(1) + \alpha_1^t\sum_j v_j(2))w_i^0(1) \\&\hspace{.5in}+ (\alpha_1^t \sum_j v_j(1) + \alpha_2^t\sum_j v_j(2))w_i^0(2) +\sum_{k\neq 1,2} \sum_j v_j(k) w_i^0(k)\Big).
\end{align*}
We use Hoeffding concentration inequality for bounding $\sum_j v_j(k) = O(\sqrt{n}\log(d)) $ uniformly for every $k\neq 1,2$. To be more concrete, by a union bound and Hoeffding it follows that uniformly over all $k\le d$ and $t\le T$:
\begin{align*}
\Pr \left(\sum_{j} v_j(k) \le \sqrt{n}\log(dT)\right)\ge 1- \exp\left(\log(dT)-2\log^2(dT)\right)=: 1-p_x.
\end{align*}
Moreover, recall that $w_i^0(k)\sim N(0,1/d)$ and :
\begin{align*}
    \Pr \left( w_i^0(k) \le\frac{\log(dm)}{\sqrt{d}}\right) \ge 1- \frac{\exp(-\frac{1}{2}\log^2(dm))}{\sqrt{2\pi} \log(dm)}
\end{align*}
By union bound we have for all $i\le m, k\le d$:
\begin{align*}
    \Pr \left( w_i^0(k) \le\frac{\log(dm)}{\sqrt{d}}\right) \ge 1- \frac{\exp(-\frac{1}{2}\log^2(dm)+\log(dm))}{\sqrt{2\pi} \log(dm)}=:1-p_w.
\end{align*}
Therefore w.p. $1-p_x-p_w$ we have for all $t\le T, i\le m,k\le d$: $\sum_j v_j(k)w_i^0(k)\le \frac{\log(dm)\log(dT)\sqrt{n}}{\sqrt{d}}$
By Hoeffding w.p. $1-p_x-p_w$ over data sampling and initialization: 
\begin{align*}
    \Pr\left(\sum_k\sum_j v_j(k)w_i^0(k)\le \log(dm)\log(dT)\log(d)\sqrt{n}\right) \ge 1-\exp(-\log^2(d))=:1-p_1.
\end{align*}
Similarly, we obtain $\Pr\left(\sum_j v_j(1) w_i^0(1)\le \frac{\sqrt{n}}{\sqrt{d}}\log(dT)\log(dm)\right)\ge 1-p_x-p_w$
Overall, the above calculations show that for some constant $C$ with probability $1-C(p_1-p_x-p_w)$ over initialization and data sampling it holds uniformly over all $k,i,t$:
\begin{align*}
  |\nabla  F_i^k (\E_x[w^t])| & \lesssim \frac{\log(dT)\log(dm)}{nm} \left(\frac{2\sqrt{n}(|\alpha_1^t|+\alpha_2^t)}{\sqrt{d}}+\sqrt{n}\log(d)\right) \\
&= \log(dT)\log(dm)\frac{|\alpha_1^t|+\alpha_2^t+\sqrt{d}\log(d)}{m\sqrt{nd}}.
\end{align*}

\paragraph{Case 2:}Now assume $k=1$ then in a similar way as above
\begin{align*}
-\nabla F_i^1(\E_x[w^t]) &= \frac{a_i}{nm} \sum_{j=1}^n \langle x_j,\E_x[w_i^t]\rangle x_j(2)\\
&= \frac{a_i}{m} \E_x[w_i^t(2)] + \frac{a_i}{nm} \sum_{j=1}^n \sum_{j'\neq 2}^d \E_x[w_i^t(j')] x_j(j')x_j(2)\\
&= \frac{a_i}{m} \E_x[w_i^t(2)] + O(\log(dm)\log(dT)\frac{|\alpha_1^t|+\alpha_2^t+\sqrt{d}\log(d)}{m\sqrt{nd}}).
\end{align*}
for $k=2:$

\bea
-\nabla F_i^2(\E_x[w^t]) = \frac{a_i}{m} \E_x[w_i^t(1)] + O(\log(dm)\log(dT)\frac{|\alpha_1^t|+\alpha_2^t+\sqrt{d}\log(d)}{m\sqrt{nd}}).
\eea
Therefore for the weights entering the $i$'th hidden neuron and for all $i\le m,t\le T$, the following concentration bound holds w.p. $1-C(p_1-p_x-p_w)$:
\begin{align*}
\left \|\nabla F_i(\E_x[w^t])- \E_x[\nabla F_i(\E_x[w^t])] \right\| \lesssim \log(dm)\log(dT)\frac{|\alpha_1^t|+\alpha_2^t+\sqrt{d}\log(d)}{m\sqrt{n}}.
\end{align*}
Therefore, as a result of Term II calculations, we deduce that with probability $1-C(p_1-p_x-p_w)$, 
\begin{align*}
\nn \forall i,t: \left \|\nabla F_i(\E_x[w^t])- \E_x[\nabla F_i(\E_x[w^t])] \right\| \lesssim \log(dm)\log(dT)\frac{|\alpha_1^t|+\alpha_2^t+\sqrt{d}\log(d)}{m\sqrt{n}},
\end{align*}
where $x^t$ denotes the data chosen at iteration $t$ of SGD and recall $w^0$ denotes the initialization weights.
\subsection{Term I:}
\paragraph{Case 1:} We first consider $k\neq 1,2$. By applying Hoeffding's concentration inequality we obtain \\$\sum_{j=1}^n x_j(\ell)x_j(k)x_j(1)x_j(2)<\sqrt{n}\log(dT)$ uniformly for every $\ell\in[d],t\le T$ w.p. $1-\exp(\log(dT)-2\log^2(dT))=:1-p_x.$ Applying Hoeffding again yields,
\begin{align*}
\nabla F_i^k(w^t) - \nabla F_i^k(\E_x[w^t]) &= \frac{a_i}{nm}\sum_{j=1}^n \langle x_j,\E_x[w_i^t]-w_i^t\rangle x_j(k)x_j(1)x_j(2)\\
&= \frac{a_i}{nm}\left\langle \sum_{j=1}^n x_j x_j(k)x_j(1)x_j(2),\E_x[w_i^t]-w_i^t\right\rangle\\
&\le \frac{\|\E_x[w_i^t]-w_i^t\|}{\sqrt{n}m} \log(dT)\log(dmT),
\end{align*}
w.p. $1-\exp(\log(dmT)-2\log^2(dmT))-\exp(\log(dT)-2\log^2(dT))=:1-p_2-p_x$ uniformly over all $k\le d, i\le m, t\le T$ where in the above we assumed $\E_x[w_i^t]-w_i^t$ as a fixed vector.  

Denote $v^t := \E_x[w_i^t]-w_i^t$. Concretely, we have the following bound for $k\neq 1,2$ on Term I:

\begin{align*}
\Pr_{x^t}\left (\forall i,k,t : \nabla F_i^k(w^t) - \nabla F_i^k(\E_x[w^t]) \le \frac{\|v^t\|}{\sqrt{n}m} \log(dT)\log(dmT) \Big| v^t \right) \ge 1-p_2-p_x.
\end{align*}
Therefore,
\begin{align*}
\Pr_{x^t,v^t} \left(\forall i,k,t : \nabla F_i^k(w^t) - \nabla F_i^k(\E_x[w^t]) \le \frac{\|v^t\|}{\sqrt{n}m} \log(dT)\log(dmT) \right) \ge 1-p_2-p_x.
\end{align*}

\paragraph{Case 2:} For $k=1:$
\begin{align*}
\nabla F_i^k(w^t) - \nabla F_i^k(\E_x[w^t]) &= \frac{a_i}{nm}\sum_{j=1}^n \langle x_j,\E_x[w_i^t]-w_i^t\rangle x_j(2)\\
&= \frac{a_i}{nm} \langle \sum_{j=1}^n x_jx_j(2),\E_x[w_i^t]-w_i^t\rangle\\
&= \frac{a_i}{m} v^t(2) + \frac{a_i}{m\sqrt{n}}\|v^t(\ell\neq 2)\|\cdot\log(dT)\log(mT),
\end{align*}
w.p. $1-p_2-p_x$. 

Similarly for $k=2$ with the same probability it holds,
\begin{align*}
\nabla F_i^k(w^t) - \nabla F_i^k(\E_x[w^t]) &= \frac{a_i}{nm}\sum_{j=1}^n \langle x_j,\E_x[w_i^t]-w_i^t\rangle x_j(1)\\
&= \frac{a_i}{nm} \langle \sum_{j=1}^n x_jx_j(1),\E_x[w_i^t]-w_i^t\rangle\\
&= \frac{a_i}{m} v^t(1) + \frac{a_i}{m\sqrt{n}}\|v^t(\ell\neq 1)\|\cdot\log(dT) \log(mT)
\end{align*}

\subsection{Combining two terms}
%\paragraph{Case 1}: $a_i=-1$: 
Recall we defined $v^t = \E_x[w_i^t]-w_i^t$ where $v^t\in\R^d$. Denote its entries by $v^t(k)$ for $k\le d$.  
Define $v^0=0$. Then, from the last two sections we obtain the following system of equations which hold w.p. $1-C(p_1-p_x-p_2-p_w)$ for all $t\ge 1$:
\begin{align*}
|v^t(1)|&\le \eta \sum_{\tau=0}^{t-1}\frac{|v^\tau(2)|}{m}+ \frac{\|v^\tau(\ell\neq 2)\|\log(dT)\log(mT)}{m\sqrt{n}} +\frac{|\alpha_1^\tau|+\alpha_2^\tau+\sqrt{d}\log(d)}{m\sqrt{nd}}\log(dm)\log(dT)\\
|v^t(2)|&\le \eta \sum_{\tau=0}^{t-1}\frac{|v^\tau(1)|}{m}+ \frac{\|v^\tau(\ell\neq 1)\|\log(dT)\log(mT)}{m\sqrt{n}} +\frac{|\alpha_1^\tau|+\alpha_2^\tau+\sqrt{d}\log(d)}{m\sqrt{nd}}\log(dm)\log(dT)\\
|v^t(k)|&\le \eta \sum_{\tau=0}^{t-1}\frac{\|v^\tau\|\log(dT)\log(dmT)}{m\sqrt{n}} +\frac{|\alpha_1^\tau|+\alpha_2^\tau+\sqrt{d}\log(d)}{m\sqrt{nd}}\log(dm)\log(dT),\;\;\; \forall k\neq 1,2
\end{align*}

Note that $|\alpha_1^\tau|+\alpha_2^\tau = 2^\tau$. Simplify the above by assuming $\eta = m$ and denoting $\gamma_\tau:=\frac{2^\tau+\sqrt{d}\log(d)}{\sqrt{nd}}\log(dm)\log(dT)$. Then the equations above simplify to: 

\begin{align*}
|v^t(1)|&\le \sum_{\tau=0}^{t-1}|v^\tau(2)|+ \frac{\|v^\tau\|\log^2(dmT)}{\sqrt{n}} +\sum_{\tau=0}^{t-1}\gamma_\tau\\
|v^t(2)|&\le  \sum_{\tau=0}^{t-1}|v^\tau(1)|+ \frac{\|v^\tau\|\log^2(dmT)}{\sqrt{n}} +\sum_{\tau=0}^{t-1}\gamma_\tau\\
|v^t(k)|&\le   \sum_{\tau=0}^{t-1}\frac{\|v^\tau\|\log^2(dmT)}{\sqrt{n}} +\sum_{\tau=0}^{t-1}\gamma_\tau,\;\;\; \forall k\neq 1,2
\end{align*}

Define $z^t:=\max_k |v^t(k)|$. Then since the RHS in the third equations is smaller than the two first equations we deduce that:
\bea\label{eq:beflem}
z^t &\le \sum_{\tau=0}^{t-1}{z^\tau}+ \frac{z^\tau \sqrt{d}\log^2(dmT)}{\sqrt{n}} +\frac{2^\tau+\sqrt{d}\log(d)}{\sqrt{nd}}\log^2(dmT)
\eea

To proceed in simplifying the equation above, we need the following lemma which can be straight-forwardly proved by induction. 
\begin{lemma}\label{lem:recu} Let $v_k,\beta,\gamma_k\in \R$. If $v_0=0$ and for every $t<T$:
\bea
v_{t+1} \le \beta \sum_{k=0}^t v_k + \sum_{k=0}^t \gamma_k
\eea

Then, it can be checked that this results in:
\bea
v_{T} \le \sum_{k=0}^{T-1} (\beta+1)^k \gamma_{T-k-1}.
\eea
\end{lemma}
Therefore, continuing from Eq. \ref{eq:beflem}:
\begin{align*}
z^t &\le \sum_{\tau=0}^{t-1} (2+\frac{\sqrt{d}\log^2(dmT)}{\sqrt{n}})^\tau \cdot\frac{2^{t-\tau-1}+\sqrt{d}\log(d)}{\sqrt{nd}}\log^2(dmT)\\
&\le \frac{t\log^2(dmT)}{\sqrt{nd}} (2+\frac{\sqrt{d}\log^2(dmT)}{\sqrt{n}})^{t-1} + \frac{(2+\frac{\sqrt{d}\log^2(dmT)}{\sqrt{n}})^t \log^3(dmT)}{\sqrt{n}(1+\frac{\sqrt{d}\log^2(dmT)}{\sqrt{n}})}
\end{align*}
Assume $t\le \log(d)$ and $n\ge d\cdot \log^{2c}(d)$ for some constant $c$. Then,
\begin{align*}
z^t &\le \frac{d}{d\log^{c-3}(d)} + \frac{2d}{\sqrt{d}\log^{c-3}(d)}\\
&\le \frac{3\sqrt{d}}{\log^{c-3}(d)}
\end{align*}
where in the above we used the fact that for $t\le T\le\log(d)$ and any $m = O(\poly(d))$:
\begin{align*}
(2+\frac{\sqrt{d}\log^2(dmT)}{\sqrt{n}})^{t-1} = (2+\frac{1}{\log^{c-2}(d)})^{t-1} \le (2+\frac{1}{\log(d)})^{t-1} \le d
\end{align*}
Therefore 
\begin{align*}
\max_k |v^t(k)| = O(\frac{\sqrt{d}}{\log^{c-3}(d)})
\end{align*}
We next show that $\|v^t(3:d)\|$ is much smaller. To see this, we replace our derived bound in the third equation above. define $\beta^t:=\|v^t(3:d)\|$ and note that $\beta^t \le \sqrt{d} |v^t(k)|$ for some $k\in[3,d]$. Thus:

\begin{align*}
\beta^t &\le \sum_{\tau=0}^{t-1} \frac{\|v^\tau\|\log^2(dmT)\sqrt{d}}{\sqrt{n}} + \sqrt{d}\sum_{\tau=0}^{t-1}\gamma_\tau\\
&\le \sum_{\tau=0}^{t-1} \frac{\|v^\tau\|}{\log^{c-2}(d)} + \sqrt{d}\sum_{\tau=0}^{t-1}\gamma_\tau\\
&\le  \sum_{\tau=0}^{t-1} \frac{\sqrt{({\beta^\tau})^2+\frac{d}{\log^{2c-6}(d)}}}{\log^{c-2}(d)} + \sqrt{d}\sum_{\tau=0}^{t-1}\gamma_\tau\\
&\le \sum_{\tau=0}^{t-1} \frac{{\beta^\tau}+\frac{\sqrt{d}}{\log^{c-3}(d)}}{\log^{c-2}(d)} + \sum_{\tau=0}^{t-1}\frac{2^\tau+\sqrt{d}\log(d)}{\sqrt{n}}\log^2(dmT)
\end{align*}
Using the lemma for the recursive summation (Lemma \ref{lem:recu}) again:
\begin{align*}
\beta^t &\le \sum_{\tau=0}^{t-1} (1+\frac{1}{\log^{c-2}(d)})^\tau(\frac{\sqrt{d}}{\log^{2c-5}(d)}+\frac{2^{t-\tau-1}+\sqrt{d}\log(d)}{\sqrt{n}}\log^2(dmT))\\
&\le (1+\frac{1}{\log^{c-2}(d)})^t \frac{\sqrt{d}}{\log^{c-3}(d)} + \frac{2^{t-1}t}{\sqrt{d}\log^{c-2}(d)} + (1+\frac{1}{\log^{c-2}(d)})^{t}\log^2(dmT)
\end{align*}
where in the above we used the conditions on $m,T$. By using them again and recalling $t\le T\le \log(d)$, the above simplifies into:
\begin{align*}
    \beta^t &\le  \frac{2\sqrt{d}}{\log^{c-3}(d)} + \frac{\sqrt{d}}{\log^{c-3}(d)} + 2\log^2(dmT)\\
    &= O( \frac{\sqrt{d}}{\log^{c-3}(d)})
\end{align*}

recall that $\E_x[w_i^t]$ had the following form 
\begin{align*}
\E_x[w_i^t(1)] &= \left(\sum_{r\in[t],r: even} \binom{t}{r} (\frac{\eta a_i}{m})^r\right) w_i^0(1)  + \left(\sum_{r\in[t],r: odd} \binom{t}{r} (\frac{\eta a_i}{m})^r\right) w_i^0(2) \\
&= \frac{1}{2}((1+a_i)^t+(1-a_i)^t) w_i^0(1) + \frac{1}{2}((1+a_i)^t-(1-a_i)^t) w_i^0(2)\\
\E_x[w_i^t(2)] &= \left(\sum_{r\in[t],r: even} \binom{t}{r} (\frac{\eta a_i}{m})^r\right) w_i^0(2)  + \left(\sum_{r\in[t],r: odd} \binom{t}{r} (\frac{\eta a_i}{m})^r\right) w_i^0(1)\\
&= \frac{1}{2}((1+a_i)^t+(1-a_i)^t) w_i^0(2) + \frac{1}{2}((1+a_i)^t-(1-a_i)^t) w_i^0(1)\\
\E_x[w_i^t(k)] &= w_i^0(k) , k>2.
\end{align*}

where the last line is derived by choosing w.l.o.g. $\eta = m:$
\begin{align*}
\E_x[w_i^t(1)] &= 2^{t-1}(w_i^0(1) + a_i w_i^0(2)),\\
\E_x[w_i^t(2)] &= 2^{t-1}(w_i^0(2) + a_i w_i^0(1)),\\
\E_x[w_i^t(k)] &= w_i^0(k), \,\, k>2.
\end{align*}
therefore,
\begin{align*}
|\E_x[w_i^t(1)]| &= \Theta(\frac{2^t}{\sqrt{d}}),\\
|\E_x[w_i^t(2)]| &= \Theta(\frac{2^t}{\sqrt{d}}),\\
|\E_x[w_i^t(k)]| &= \Theta(\frac{1}{\sqrt{d}}).
\end{align*}

Select $t=\log(d)$. Then, w.h.p. it holds that
\begin{align*}
|\E_x[w_i^t(1)]| &= \Theta(\sqrt{d}),\\
|\E_x[w_i^t(2)]| &= \Theta(\sqrt{d}),\\
|\E_x[w_i^t(k)]| &= \Theta(\frac{1}{\sqrt{d}}).
\end{align*}

The calculations above indicate that the signal strength is already larger than the noise after $\log(d)$ SGD steps. We are now ready to compute $y\Phi(w^t,x)$. Assume for simplicity and without loss of generality that $x(1)=x(2)=y=1$ as other cases lead to the same result. Recall $v_i^t(k) := \E_x[w_i^t(k)]-w_i^t(k)$. Then,

\begin{align*}
y\Phi(w^t,x) &= \frac{1}{m}\sum_{i=1}^m a_i\left(w_i^t(1)+w_i^t(2)+\sum_{k=3}^d w_i^t(k)x(k)\right)^2\\
&=\frac{1}{m}\sum_{i=1}^m a_i \left(d(w_i^0(1)+w_i^0(2))(a_i+1)-v_i^t(1) - v_i^t(2)+\sum_{k\ge 3} (w_i^0(k)-v_i^t(k))x(k)  \right)^2\\
&= \frac{1}{m}\sum_{i\in P} \left(2d(w_i^0(1)+w_i^0(2))-v_i^t(1) - v_i^t(2)+\sum_{k\ge 3} (w_i^0(k)-v_i^t(k))x(k)  \right)^2\\
&\hspace{.5in}+\frac{1}{m}\sum_{i\in N}\left(v_i^t(1) + v_i^t(2)+\sum_{k\ge 3} (v_i^t(k)-w_i^0(k))x(k)  \right)^2,
  %  &= \frac{2^{2t-1} }{m} \sum_{i=1}^m (w_i^0(1)+w_i^0(2))^2(a_i+1) + \frac{1}{m} \sum_{i=1}^m a_i\frac{c^{2t}}{d} + \frac{1}{m} \sum_{i=1}^m a_i (\sum_{k\neq 1,2} (w_i^0(k)+\frac{c^t}{\sqrt{d}})x(k) )^2
\end{align*}

where $P$ and $N$ denote the set of neurons for which $a_i=1$ and $a_i=-1$, respectively.

Note that the last term in each summation can be simplified as follows: $$\sum_{k\ge 3} (w_i^0(k)+v_i^t(k)) x(k) = O(\|v_i^t\|) = O\left(\beta^t\right) = O\left(\frac{\sqrt{d}}{\log^{c-3}(d)}\right).$$
Concretely, by Hoeffding w.p. $1-C(p_1-p_2-p_x-p_w)$ over data sampling and initialization :
\begin{align}
    \Pr_x\left(\sum_{k\ge 3} \left(w_i^0(k)+v_i^t(k)\right) x(k) > \frac{\sqrt{d}}{\log^{\frac{c-4}{2}}(d)}\right) \le e^{-\log(d)} = \frac{1}{d}.
\end{align}

Define $\zeta_i^t:=v_i^t(1) + v_i^t(2)+\sum_{k\ge 3} (v_i^0(k)-w_i^t(k))x(k)$. Then,
\begin{align}
y\Phi(w^t,x) &= \frac{1}{m}\sum_{i\in P}\left(2d(w_i^0(1)+w_i^0(2))-\zeta_i^t\right)^2 + \frac{1}{m}\sum_{i\in N} \left(\zeta_i^t\right)^2\nn\\
&= \frac{1}{m}\sum_{i\in P} 4d^2 (w_i^0(1)+w_i^0(2))^2 -4d\zeta_i^t (w_i^0(1)+w_i^0(2)) + \frac{1}{m}\sum_{i\in[m]} (\zeta_i^t)^2\nn\\
&\ge \frac{1}{m}\sum_{i\in P} 4d^2 (w_i^0(1)+w_i^0(2))^2-4d|\zeta_i^t|\cdot |w_i^0(1)+w_i^0(2)|.\label{eq:phi}
\end{align}
Moreover, we had w.p. at least $1-C(p_1+p_2+p_x+p_w)$ that $|v_i^t(1)|+|v_i^t(2)| < \frac{C\sqrt{d}}{\log^{c-3}(d)}$ for some absolute constant $C.$ Adding these two together, we deduce that w.p. $1-o_d(1)$ it holds uniformly over $i,t$,
\begin{align*}
|\zeta_i^t| < \frac{C\sqrt{d}}{\log^{\frac{c-4}{2}}(d)}.
\end{align*}
It remains to obtain lower- and upper-bounds on  $(w_i^0(1)+w_i^0(2))^2$ as required by Eq. \ref{eq:phi}. To proceed, note that $\frac{d}{2}(w_i^0(1)+w_i^0(2))^2\sim \chi^2(1).$ Therefore with standard concentration bounds for $\chi^2(1)$ and $\chi^2(m)$ distributions (e.g., \cite{laurent2000adaptive}) we find for any $u\ge 0$,
\begin{align*}
&\Pr\left(\frac{d}{2}(w_i^0(1)+w_i^0(2))^2-1\ge 2\sqrt{u}+2u\right)\le e^{-u}\\
&\Pr\left(m-\frac{d}{2}\sum_{i=1}^m(w_i^0(1)+w_i^0(2))^2\ge 2\sqrt{mu}\right)\le e^{-u}
\end{align*}
By selecting $u=\log^2(d)$ in the first inequality above, we find w.p. at most $e^{\log(m)-\log^{2}(d)}$ it holds uniformly for all $i\le m$:
\begin{align*}
&(w_i^0(1)+w_i^0(2))^2\ge \frac{2}{d}+ \frac{4}{d}\sqrt{\log^{2}(d)}+\frac{4}{d}\log^{2}(d),
\end{align*}
and by noting that $|P| = |N| = m/2$ and selecting $u=m/16$ in the second inequality we have w.p. at least $1-e^{-\frac{m}{16}}$,
\begin{align*}
\frac{1}{m}\sum_{i\in P} (w_i^0(1)+w_i^0(2))^2 \ge (1-\frac{1}{\sqrt{2}})\frac{1}{d}.
\end{align*}

Hence w.p. $1-e^{\log(m)-\log^2(d)} - e^{-\frac{m}{16}}-C(e^{-\log^2(d)}+e^{\log(dmT)-2\log^2(dmT)}+e^{\log(dT)-2\log^2(dT)}+e^{-\frac{1}{2}\log^2(dm)+\log(dm)})$ over data sampling and initialization,
\bea\label{eq:phifinal}
\Pr_{x,y} \left(y\Phi(w^t,x) \ge d- \frac{4Cd}{\log^{\frac{c-6}{2}}(d)}\right)\ge 1-1/d.
\eea

%and the model's output with probability at least $1-\frac{1}{d}$ (over $x$) is :

% \begin{align}
%     &\frac{1}{m}\sum_{i=1}^m a_i \left(d(w_i^0(1)+w_i^0(2))(a_i+1)+O(\frac{\sqrt{d}}{\log^{c/2}(d)}) \right)^2\\
%     &\approx \frac{d^2}{d} + \frac{d^2 }{m}\sum_{i=1}^m a_i (w_i^0(1)+w_i^0(2))^2 +O(\frac{d}{\log^{c}(d)}) + O(\frac{d}{\log^{c/2}(d)})\\
%     &= d + \tilde O\left(\frac{d}{\log^{c/2}(d)} + \frac{d}{\sqrt{m}}\right)
% \end{align}
For the theorem's choice of $T=\log(d),m = O(\poly(d))$ it is deduced that $e^{\log(m)-\log^2(d)} + C(e^{-\log^2(d)}+e^{\log(dmT)-2\log^2(dmT)}+e^{\log(dT)-2\log^2(dT)}+e^{-\frac{1}{2}\log^2(dm)+\log(dm)}) = o_d(1)$. Moreover, for the theorem's choice $c=7$ and large enough $d$ it holds that $d-4Cd/\log^{\frac{c-6}{2}(d)}\ge 0.$ This leads to the following lower bound on test accuracy,
\bea
\Pr_{x,y} \left(y\Phi(w^t,x) \ge 0\right)\ge 1-\frac{1}{d},
\eea
with probability at least  $1 - e^{-\frac{m}{16}}-e^{\log(m)-\log^2(d)}-o_d(1)$ over data sampling and initialization. This completes the proof of Theorem \ref{thm:xor}. 

\end{document}